\documentclass[11pt]{article}
\usepackage{graphicx,amsmath,amsfonts,amssymb,bm,hyperref,url,breakurl,epsfig,epsf,color,fullpage,MnSymbol,mathbbol,fmtcount,algorithmic,algorithm,semtrans,cite,caption,subcaption,
  multirow}

\usepackage{caption}
\usepackage[bottom,hang,flushmargin]{footmisc}
\usepackage{comment}
\usepackage[bottom]{footmisc}
\usepackage{caption}
\usepackage{subcaption}

\newcommand{\bi}{\begin{itemize}}
\newcommand{\ei}{\end{itemize}}
\newcommand{\bal}{\begin{align}}
\newcommand{\eal}{\end{align}}

\newcommand{\EE}{\mathbb{E}}
\newcommand{\PP}{\mathbb{P}}

\newcommand{\bX}{\mathbf{X}}
\newcommand{\bY}{\mathbf{Y}}

\newcommand{\bW}{\mathbf{W}}
\newcommand{\bx}{\mathbf{x}}
\newcommand{\by}{\mathbf{y}}
\newcommand{\bR}{\mathbf{R}}
\newcommand{\bp}{\mathbf{p}}
\newcommand{\bw}{\mathbf{w}}
\newcommand{\ba}{\mathbf{a}}
\newcommand{\bv}{\mathbf{v}}

\newcommand{\bu}{\mathbf{u}}
\newcommand{\bs}{\mathbf{s}}

\newcommand{\bA}{\mathbf{A}}
\newcommand{\bB}{\mathbf{B}}

\newcommand{\bU}{\mathbf{U}}
\newcommand{\bV}{\mathbf{V}}

\newcommand{\bQ}{\mathbf{Q}}

\newcommand{\bN}{\mathbf{N}}
\newcommand{\bK}{\mathbf{K}}
\newcommand{\bD}{\mathbf{D}}
\newcommand{\bE}{\mathbf{E}}
\newcommand{\bF}{\mathbf{F}}
\newcommand{\bI}{\mathbf{I}}

\newcommand{\cX}{\mathcal{X}}
\newcommand{\cY}{\mathcal{Y}}

\newcommand{\cO}{\mathcal{O}}

\newcommand{\cG}{\mathcal{G}}

\newcommand{\hbx}{\hat{\mathbf{x}}}
\newcommand{\rhoh}{\hat{\rho}}
\newcommand{\phih}{\hat{\phi}}

\newcommand{\tI}{\tilde{I}}
\newcommand{\hphi}{\hat{\phi}}

\newcommand{\bb}{\mathbf{b}}

\newcommand{\cK}{\mathcal{K}}
\newcommand{\btheta}{\mathbf{\theta}}

\newcommand{\eps}{\epsilon}

\usepackage{hyperref}
\definecolor{darkred}{RGB}{150,0,0}
\definecolor{darkgreen}{RGB}{0,150,0}
\definecolor{darkblue}{RGB}{0,0,200}
\hypersetup{colorlinks=true, linkcolor=darkred, citecolor=darkgreen, urlcolor=darkblue}

\numberwithin{equation}{section}

\def \endprf{\hfill {\vrule height6pt width6pt depth0pt}\medskip}
\newenvironment{proof}{\noindent {\bf Proof} }{\endprf\par}

\newtheorem{theorem}{\textbf{Theorem}}
\newtheorem{lemma}{\textbf{Lemma}}
\newtheorem{corollary}{\textbf{Corollary}}

\newtheorem{definition}{\textbf{Definition}}

\newtheorem{remark}{\textbf{Remark}}

\newtheorem{proposition}{\textbf{Proposition}}

\title{{\huge Maximally Correlated Principal Component Analysis}}

\author{Soheil~Feizi and David Tse\\\\
Stanford University}
\date{}

\begin{document}
\maketitle

\begin{abstract}
In the era of big data, reducing data dimensionality is critical in many areas of science. Widely used Principal Component Analysis (PCA) addresses this problem by computing a low dimensional data embedding that maximally explain variance of the data. However, PCA has two major weaknesses. Firstly, it only considers linear correlations among variables (features), and secondly it is not suitable for categorical data. We resolve these issues by proposing Maximally Correlated Principal Component Analysis (MCPCA). MCPCA computes transformations of variables whose covariance matrix has the largest Ky Fan norm. Variable transformations are unknown, can be nonlinear and are computed in an optimization. MCPCA can also be viewed as a multivariate extension of Maximal Correlation. For jointly Gaussian variables we show that the covariance matrix corresponding to the identity (or the negative of the identity) transformations majorizes covariance matrices of non-identity functions. Using this result we characterize global MCPCA optimizers for nonlinear functions of jointly Gaussian variables for every rank constraint. For categorical variables we characterize global MCPCA optimizers for the rank one constraint based on the leading eigenvector of a matrix computed using pairwise joint distributions. For a general rank constraint we propose a block coordinate descend algorithm and show its convergence to stationary points of the MCPCA optimization. We compare MCPCA with PCA and other state-of-the-art dimensionality reduction methods including Isomap, LLE, multilayer autoencoders (neural networks), kernel PCA, probabilistic PCA and diffusion maps on several synthetic and real datasets. We show that MCPCA consistently provides improved performance compared to other methods.
\end{abstract}

\section{Introduction}\label{sec:intro}
Let $X_1$ and $X_2$ be two mean zero and unit variance random variables. Pearson's correlation \cite{pearson1895note} defined as
\begin{align}\label{eq:corr}
\rho_{Pearson}(X_1,X_2)=\EE[X_1 X_2]
\end{align}
is a basic statistical parameter and plays a central role in many statistical and machine learning methods such as linear regression \cite{neter1996applied}, principal component analysis \cite{jolliffe2002principal}, and support vector machines \cite{steinwart2008support}, partially owing to its simplicity and computational efficiency. Pearson's correlation however has two main weaknesses: firstly it only captures linear dependency between variables, and secondly for discrete (categorical) variables the value of Pearson's correlation depends somewhat arbitrarily on the labels. To overcome these weaknesses, {\it Maximal Correlation} (MC) has been proposed and studied by Hirschfeld \cite{hirschfeld1935connection}, Gebelein \cite{gebelein1941statistische}, Sarmanov \cite{sarmanov1962maximum} and R\'enyi \cite{renyi1959measures}, and is defined as
\begin{align}\label{eq:mc}
\rho_{MC}(X_1, X_2)=
	\sup_{\phi_1(.), \phi_{2}(.)}\ &\EE[\phi_1(X_1)\ \phi_{2}(X_{2})],\\
& \EE[\phi_i(X_i)]=0,\quad i=1,2,\nonumber\\
& \EE[\phi_i(X_i)^2]=1,\quad i=1,2.\nonumber
\end{align}
Transformation functions $\{\phi_i(.)\}_{i=1}^{2}$ are assumed to be Borel measurable whose ranges are in $\mathbb{R}$. MC has also been studied by Witsenhausen \cite{witsenhausen1975sequences}, Ahlswede and G\'acs \cite{ahlswede1976spreading}, and Lancaster \cite{lancaster1957some}. MC tackles the two main drawbacks of the Pearson's correlation: it models a family of nonlinear relationships between the two variables. For discrete variables, the MC value only depends on the joint distribution and does not rely on labels. Moreover the MC value between $X_1$ and $X_2$ is zero iff they are independent \cite{renyi1959measures}.

For the multivariate case with variables $X=(X_1$,...,$X_p)^T$ where $p\geq 2$, Pearson's correlation can be extended naturally to the covariance matrix $\bK_X \in \mathbb{R}^{p\times p}$ where $\bK_X(i,i')=\EE[X_i X_{i'}]$ (assuming $X_i$ has zero mean and unit variance). Similarly to the bivariate case, the covariance matrix analysis suffers from two weaknesses of only capturing linear dependencies among variables and being label dependent when variables are discrete (categorical). One way to extend the idea of MC to the multivariate case is to consider the set of covariance matrices of transformed variables. Let $\phi(X)=(\phi_1(X_1),...,\phi_p(X_p))^T$ be the vector of transformed variables with zero means and unit variances. I.e., $\EE[\phi_i(X_i)]=0$ and $\EE[\phi_i(X_i)^2]=1$ for $1\leq i\leq p$. Let $\bK_{\phi(X)}\in \mathbb{R}^{p\times p}$ be the covariance matrix of transformed variables $\phi(X)$ where $\bK_{\phi(X)}(i,i')=\EE[\phi_i(X_i) \phi_{i'}(X_{i'})]$. The set of covariance matrices of transformed variables is defined as follows:

\begin{align}\label{eq:MC-covariance-set}
\cK_{X}\triangleq \Big\{\bK_{\phi(X)}\in \mathbb{R}^{p\times p}:~\EE[\phi_i(X_i)]=0,~\EE[\phi_i(X_i)^2]=1,~ 1\leq i,i'\leq p \Big\}.
\end{align}

Similarly to the bivariate case, functions $\{\phi_i(.)\}_{i=1}^p$ are assumed to be Borel measurable whose ranges are in $\mathbb{R}$. If variables $\{X_i\}_{i=1}^{p}$ are continuous, functions $\{\phi_i(.)\}_{i=1}^{p}$ are assumed to be continuous. The set $\cK_{X}$ includes infinitely many covariance matrices corresponding to different transformations of variables. In order to have an operational extension of MC to the multivariate case, we need to select one (or finitely many) members of $\cK_{X}$ through an optimization.

\begin{figure}
\centering
  \includegraphics[width=0.8\linewidth]{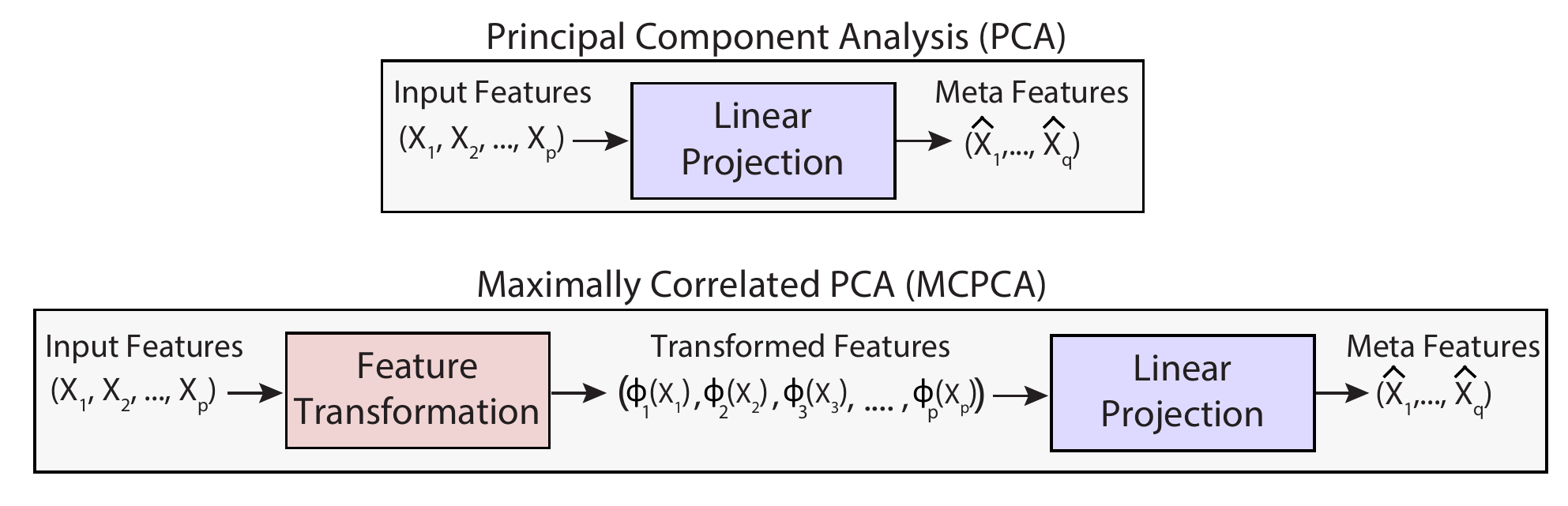}
\caption{An illustration of the Maximally Correlated Principal Component Analysis (MCPCA) Framework. MCPCA computes transformations of features $\phi_i(X_i)$ so that variance of the transformed data can be explained maximally by a few meta features. Feature transformations are unknown, can be nonlinear and are computed in an optimization.}
\label{fig:MCPCA}
\end{figure}

Here we propose the following optimization over $\cK_{X}$ that aims to select a covariance matrix $\bK^*\in \cK_{X}$ with the maximum $q$-Ky Fan norm (i.e., with the maximum sum of top $q$ eigenvalues):
\begin{align}\label{opt:MCPCA-main}
\max_{\bK}\quad & \sum_{r=1}^{q} \lambda_r(\bK)\\
& \bK\in \cK_{X}.\nonumber
\end{align}
Since the trace of all matrices in $\cK_{X}$ is equal to $p$, maximizing the Ky Fan norm over $\cK_{X}$ results in a low rank or an approximately low rank covariance matrix. We refer to this optimization as {\it Maximally Correlated Principal Component Analysis} with parameter $q$ or for simplicity, the MCPCA optimization. The optimal MCPCA value is denoted by $\rho_q^*(X)$. When no confusion arises we use $\rho_q^*$ to refer to it.

Principal Component Analysis (PCA) \cite{jolliffe2002principal} aims to find $q$ eigenvectors corresponding to the top eigenvalues of the covariance matrix. These are called Principal Components (PCs). On the other hand, we show that the MCPCA optimization aims to find possibly nonlinear transformations of variables that can be approximated optimally by $q$ orthonormal vectors. Thus, MCPCA can be viewed as a generalization of PCA over possibly nonlinear transformations of variables with zero means and unit variances.

We summarize our main contributions below:
\begin{itemize}
  \item We introduce MCPCA as a multivariate extension of MC and a generalization of PCA.
  \item   For jointly Gaussian variables we show that the covariance matrix corresponding to the identity (or the negative of the identity) transformations majorizes covariance matrices of non-identity functions. Using this result we characterize global MCPCA optimizers for nonlinear functions of jointly Gaussian variables for every $q$.
  \item For finite discrete variables,
  \begin{itemize}
    \item [-] we compute a globally optimal MCPCA solution when $q=1$ based on the leading eigenvector of a matrix computed using pairwise joint distributions.
    \item [-] for an arbitrary $q$ we propose a block coordinate descend algorithm and show its convergence to stationary points of the MCPCA optimization.
  \end{itemize}
\item We study the consistency of sample MCPCA (an MCPCA optimization computed using empirical distributions) for both finite discrete and continuous variables.
\end{itemize}

We compare MCPCA with PCA and other state-of-the-art nonlinear dimensionality reduction methods including Isomap \cite{isomap}, LLE \cite{LLE}, multilayer autoencoders (neural networks) \cite{lee2007nonlinear,hinton2006reducing}, kernel PCA \cite{scholkopf1997kernel,scholkopf1998nonlinear,hoffmann2007kernel,mika1998kernel}, probabilistic PCA \cite{roweis1998algorithms} and diffusion maps \cite{lafon2006diffusion} on several synthetic and real datasets. Our real dataset experiments include breast cancer, Parkinson’s disease, diabetic retinopathy, dermatology, gene splicing and adult income datasets. We show that MCPCA consistently provides improved performance compared to other methods.

\subsection{Prior Work}\label{subsec:prior}
MCPCA can be viewed as a dimensionality reduction method whose goal is to find possibly nonlinear transformations of variables with a low rank covariance matrix. Other nonlinear dimensionality reduction methods include manifold learning methods such as Isomap \cite{isomap}, Locally Linear Embedding (LLE) \cite{LLE}, kernel PCA  \cite{scholkopf1997kernel,scholkopf1998nonlinear,hoffmann2007kernel,mika1998kernel}, maximum variance unfolding \cite{weinberger2004learning}, diffusion maps \cite{lafon2006diffusion}, Laplacian eigenmaps \cite{belkin2001laplacian}, Hessian LLE \cite{donoho2003hessian}, Local tangent space analysis \cite{zhang2004principal}, Sammon mapping \cite{sammon1969nonlinear}, multilayer autoencoders \cite{lee2007nonlinear,hinton2006reducing}, among others. For a comprehensive review of these methods, see reference \cite{van2009dimensionality}. Although these techniques show an advantage compared to PCA in artificial datasets, their successful applications to real datasets have been less convincing \cite{van2009dimensionality}. The key challenge is to have an appropriate balance among generality of the model, computational complexity of the method and statistical significance of inferences.

MCPCA is more general than PCA since it considers both linear and nonlinear feature transformations. In kernel PCA methods, transformations of variables are {\it fixed} in advance. This is in contrast to MCPCA that optimizes over transformations resulting in an optimal low rank approximation of the data. Manifold learning methods such as Isomap and LLE aim to find a low dimensional representation of the data such that sample distances in the low dimensional space are the same, up to a scaling, to sample geodistances (i.e., distances over the manifold), assuming there exists such a manifold that the data lies on. These methods can be viewed as extensions of PCA fitting a nonlinear model to the data. Performance of these methods has been shown to be sensitive to noise and model parameters \cite{van2009dimensionality}. Through experiments on several synthetic and real datasets we show that the performance of MCPCA is robust against these factors. Note that MCPCA allows features to be transformed only individually, thus avoiding a combinatorial optimization and resulting in statistically significant inferences. However because of this MCPCA cannot capture low dimensional structures such as the swiss roll example since underlying transformation depend on pairs of variables.

Unlike existing dimensionality reduction methods that are only suitable for data with continuous features, MCPCA is suitable for both categorical and continuous data. The reason is that even if the data is categorical, transformed values computed by MCPCA are real. Moreover we compare computational and memory complexity of MCPCA and manifold learning methods (Isomap and LLE) in Remark \ref{remark:complexity}. Unlike Isomap and LLE methods whose computational and memory complexity scales in a quadratic or cubic manner with the number of samples, computational and memory complexity of the MCPCA algorithm scales linearly with the number of samples, making it more suitable for data sets with large number of samples.

MCPCA can be viewed as a multivariate extension of MC. Other extensions of MC to the multivariate case have been studied in the literature. For example, reference \cite{feizi2015network} introduces an optimization over $\cK_{X}$ that aims to maximize sum of arbitrary chosen elements of the matrix $\bK\in \cK_{X}$. \cite{feizi2015network} shows that this optimization can be useful in nonlinear regression and graphical model inference. Moreover, \cite{feizi2015network} provides an algorithm to find local optima of the proposed optimization. Reference \cite{beigi2015duality} introduces another optimization that aims to select a covariance matrix whose minimum eigenvalue is maximized. \cite{beigi2015duality} briefly discuses computational and operational aspects of the proposed optimization.

\subsection{Notation}\label{sec:notation}
For matrices we use bold-faced upper case letters, for vectors we use bold-faced lower case letters, and for scalars we use regular lower case letters. For random variables we use regular upper case letters. For example, $\bX$ represents a matrix, $\bx$ represents a vector, $x$ represents a scalar number, and $X$ represents a random variable. $I_{n}$ and $1_{n}$ are the identity and all one matrices of size $n\times n$, respectively. When no confusion arises, we drop the subscripts. $\mathbf{1}\{x=y\}$ is the indicator function which is equal to one if $x=y$, otherwise it is zero. $Tr(\bX)$ and $\bX^T$ represent the trace and the transpose of the matrix $\bX$, respectively. $diag(\bx)$ is a diagonal matrix whose diagonal elements are equal to $\bx$, while $diag(\bX)$ is a vector of the diagonal elements of the matrix $\bX$. $\|\bx\|_{2}=\bx^T \bx$ is the second norm of the vector $\bx$. When no confusion arises, we drop the subscript. $||\bX||$ is the operator norm of the matrix $\bX$. $<\bx,\by>$ is the inner product between vectors $\bx$ and $\by$. $\bx \perp \by$ indicates that vectors $\bx$ and $\by$ are orthogonal. The matrix inner product is defined as $<\bX,\bY>=Tr(\bX \bY^T)$.

The eigen decomposition of the matrix $\bX\in\mathbb{R}^{n\times n}$ is denoted by $\bX=\sum_{i=1}^{n} \lambda_i(\bX) \bu_i(\bX) \bu_i(\bX)^T$, where $\lambda_i(\bX)$ is the $i$-th largest eigenvalue of the matrix $\bX$ corresponding to the eigenvector $\bu_i(\bX)$. We have $\lambda_1(\bX)\geq \lambda_2(\bX)\geq \cdots$. $\lambda(\bX)=(\lambda_1(\bX),\lambda_2(\bX),\cdots)^T$. $\bu_i(\bX)$ has a unit norm. Similarly the singular value decomposition of the matrix $\bY\in \mathbb{R}^{n\times m}$ is denoted by $\bY=\sum_{i=1}^{min(n,m)} \sigma_{i}(\bY) \bu_i(\bY) \bv_i(\bY)^T$ where $\sigma_i(\bY)$ is the $i$-th largest singular value of the matrix $\bY$ corresponding to the left and right singular eigenvectors $\bu_i(\bY)$ and $\bv_i(\bY)$, respectively. We have $\sigma_1(\bY)\geq \sigma_2(\bY)\geq \cdots$. $\sigma(\bY)=(\sigma_1(\bY),\sigma_2(\bY),\cdots)^T$. $\bu_i(\bY)$ and $\bv_i(\bY)$ are unit norm vectors.

\section{MCPCA: Basic Properties and Relationship with Matrix Majorization}\label{sec:properties}
\subsection{Basic Properties of MCPCA}\label{sec:properties}
In reference \cite{renyi1959measures}, R\'enyi shows that MC between the two variables $X_1$ and $X_2$ is zero iff they are independent, while MC is one iff the two variables are strictly dependent (i.e., there exist mean zero, unit variance transformations of variables that are equal.). Here we study some of these properties for the multivariate case of MCPCA:

\begin{theorem}\label{thm:prop-MCPCA}
 Let $\rho_q^*$ be the optimal MCPCA value for random variables $X_1$,...,$X_p$.
\begin{itemize}
  \item [(i)] $1\leq \rho_q^*\leq p$, for $1\leq q\leq p$.
  \item [(ii)] $\rho_1^*=1$ iff $X_i$ and $X_{i'}$ are independent, for $1\leq i\neq i'\leq p$.
  \item [(iii)] $\rho_1^*=p$ iff $X_1$,...,$X_p$ are strictly dependent. I.e., there exist zero mean, unit variance transformation functions $\{\phi_i(.)\}_{i=1}^{p}$ such that for all $1\leq i,i'\leq p$, $\phi_i(X_i)=\phi_{i'}(X_{i'})$.
  \item [(iv)] If $\{\phi_i(.)\}_{i=1}^{p}$ are one-to-one transformation functions, $\rho_q^*(X)=\rho_q^*(\phi(X))$.
\end{itemize}
 \end{theorem}
\begin{proof}
To prove part (i), for any $\bK\in \cK_{X}$, we have $Tr(\bK)=\sum_{r=1}^{p} \lambda_r(\bK)=p$ because $\phi_i(X_i)$ has zero mean and unit variance for $1\leq i\leq p$. Moreover, since $\lambda_1(\bK)\geq \lambda_2(\bK)\geq \dots \lambda_p(\bK)\geq 0$, we have $\lambda_1(\bK)\geq 1$. Thus, $1\leq \rho_q^*\leq p$, for $1\leq q\leq p$. This completes the proof of part (i).

To prove part (ii), suppose $\lambda_1(\bK^*)=1$. Thus, for every $\bK\in \cK_{X}$, we have $1\geq \lambda_1(\bK)\geq \lambda_2(\bK)\geq \dots \lambda_p(\bK)\geq 0$. However since the sum of all eigenvalues are equal to $p$, we have $\lambda_i(\bK)=1$ for every $\bK\in \cK_{X}$ and $1\leq i\leq p$. Therefore, $\bK=I_p$ for every $\bK\in \cK_{X}$. This means $\rho_{MC}(X_i,X_{i'})=0$, for $1\leq i\neq i'\leq p$, which indicates that $X_i$ and $X_{i'}$ are independent \cite{renyi1959measures}. To prove the other direction of part (ii), if $X_i$ and $X_{i'}$ are independent, for every zero mean and unit variance functions $\phi_i(.)$ and $\phi_{i'}(.)$, we have $\EE[\phi_i(X_i) \phi_i(X_i)]=0$ \cite{renyi1959measures}. Thus, for every $\bK\in \cK_{X}$, we have $\bK=I_p$. This completes the proof of part (ii).

To prove part (iii), let $\rho_1^*=p$. Thus, $1_{p} \in \cK_{X}$. It means that there exist transformation functions $\{\phi_i^*\}_{i=1}^{p}$ with zero means and unit variances such that for all $1\leq i,i'\leq p$, $\EE[\phi_i^*(X_i)\phi_{i'}^*(X_{i'})]=1$. It means that for $1\leq i\leq p$, $\phi_i^*(X_i)=Y$ where $Y$ has zero mean and unit variance. The proof of the inverse direction is straightforward. This completes the proof of part (iii).

To prove part (iv), we note that if $\{\phi_i(.)\}_{i=1}^{p}$ are one-to-one transformations, $\cK_{X}=\cK_{\phi(X)}$. Thus, $\rho_q^*(X)=\rho_q^*(\phi(X))$. This completes the proof of part (iv).
\end{proof}

In the following proposition, we show that the increase ratio of the optimal MCPCA value (i.e., $\rho_{q+1}^*/\rho_q^*-1$) is bounded above by $1/q$ which decreases as $q$ increases.

\begin{proposition}\label{prop:submodular-increase}
Let $\rho_q^*$ be the optimal MCPCA value for random variables $X_1$,...,$X_p$. We have
\begin{align}\label{eq:submodular-increase}
\rho_q^* \leq \rho_{q+1}^* \leq (1+\frac{1}{q})\rho_q^*
\end{align}
\end{proposition}
\begin{proof}
Let $\bK^*$ be an optimal MCPCA solution for $q+1$. Since $\rho_q^*$ is an optimal MCPCA value with parameter $q$, we have
\begin{align}\label{eq:inequality-submodular}
\sum_{r\in\{1,2,...,q+1\}-\{j\}} \lambda_r(\bK^*) \leq \rho_q^*,\quad \forall j\in\{1,...,q+1\}.
\end{align}
By summing \eqref{eq:inequality-submodular} over all $j\in\{1,...,q+1\}$, we have $q \rho_{q+1}^*\leq (q+1) \rho_q^*$. This completes the proof.
\end{proof}
\subsection{Relationship between MCPCA and Matrix Majorization}\label{subsec:majorization}
A vector $\bx=(x_1,x_2,\cdots,x_p)^T\in\mathbb{R}^p$ {\it weakly majorizes} vector $\by=(y_1,y_2,\cdots,y_p)^T$ (in symbols, $\bx\succ_w \by$) if $\sum_{r=1}^{q}x_{[r]}\geq \sum_{r=1}^{q}y_{[r]}$, for all $1\leq q\leq p$. The symbols $x_{[1]}\geq x_{[2]}\geq \cdots \geq x_{[p]}$ stand for the elements of the vector $\bx$ sorted in a decreasing order. If $\bx\succ_w \by$ and $\sum_{r=1}^{p} x_{r}=\sum_{r=1}^{p} y_{r}$, then we say vector $\bx$ {\it majorizes} vector $\by$ and denote it by $\bx\succ \by$.

Let $\bA$ and $\bB$ be two Hermitian matrices in $\mathbb{R}^{p\times p}$. We say $\bA$ majorizes $\bB$ is $\lambda(\bA)\succ \lambda(\bB)$. We have the following equivalent formulation for matrix majorization that we will use in later parts of the paper.
\begin{lemma}\label{lem:unitary-eq}
The following conditions for Hermitian matrices $\bA$ and $\bB$ are equivalent:
\begin{itemize}
\item $\bA\prec \bB$
\item There exist unitary matrices $\bU_j$ and positive numbers $t_j$ such that
\begin{align}
\bA=\sum_{j=1}^{N} t_j \bU_j \bB \bU_j^*,
\end{align}
where $\sum_{j=1}^{N} t_j=1$.
\end{itemize}
\end{lemma}
\begin{proof}
See Theorem 7.1 in \cite{ando1989majorization}.
\end{proof}
The following proposition makes a connection between an optimal MCPCA solution and the majorization of covariance matrices in $\cK_{X}$.
\begin{lemma}\label{prop:connection-major}
If $\bK^*\in \cK_{X}$ majorizes all $\bK\in \cK_{X}$, then $\bK^*$ is an optimal solution of the MCPCA optimization \eqref{opt:MCPCA-main}, for $1\leq q\leq p$.
\end{lemma}
\begin{proof}
Since $\bK^*\in \cK_{X}$ majorizes all $\bK\in \cK_{X}$, $\sum_{r=1}^{q}\lambda_r(\bK^*)\geq \sum_{r=1}^{q}\lambda_r(\bK)$, for all $1\leq q\leq p$. Thus $\bK^*$ is an optimal solution of optimization \eqref{opt:MCPCA-main}, for $1\leq q\leq p$.
\end{proof}
\subsection{MCPCA as an Optimization over Unit Variance Functions}\label{subsec:alternative-formulations}
The feasible set of optimization \eqref{opt:MCPCA-main} includes functions of variables with zero means and unit variances. In the following we consider an alternative optimization whose feasible set includes functions of variables with unit variances and show the relationship between its optimal solutions with the ones of the MCPCA optimization. This formulation becomes useful in simplifying the MCPCA optimization for finite discrete variables (Section \ref{sec:discrete}).
\begin{lemma}\label{lem:MC-PCA-variance}
Consider the following optimization:
\begin{align}\label{opt:MC-PCA-variance}
\max_{\{\phi_i\}_{i=1}^{p}}\quad &\sum_{r=1}^{q} \lambda_{r}(\bK)\\
&\bK(i,i')=\EE\left[\left(\phi_i(X_i)-\bar{\phi_i}(X_i)\right)\left(\phi_{i'}(X_{i'})-\bar{\phi_{i'}}(X_{i'})\right)\right],\quad 1\leq i,i'\leq p\nonumber\\
&var(\phi_i(X_i))=1,\quad 1\leq i \leq p\nonumber,
\end{align}
where $var(.)$ denotes the variance of a random variables and $\bar{\phi_i}(X_i)=\EE[\phi_i(X_i)]$. Let $\varrho_1$ and $\varrho_2$ be optimal values of objective functions of optimizations \eqref{opt:MCPCA-main} and \eqref{opt:MC-PCA-variance}, respectively. We have $\varrho_1=\varrho_2$. Moreover if $\{\phi_i^{**}\}_{i=1}^{p}$ is an optimal solution of optimization \eqref{opt:MC-PCA-variance}, then $\{\phi_i^{*}\}_{i=1}^{p}$ is an optimal solution of optimization \eqref{opt:MCPCA-main}, where $\phi_i^{*}(X_i)=\phi_i^{**}(X_i)-\bar{\phi_i}^{**}(X_i)$, and vice versa.
\end{lemma}
\begin{proof}
First we have the following lemma:
\begin{lemma}\label{thm:MC-PCA-trace}
Let $(\bK^*,\bW^*)$ be an optimal solution of the following optimization:
\begin{align}\label{opt:MC-PCA-trace}
\max_{\bW, \{\phi_i\}_{i=1}^{p}}\quad &Tr(\bW \bK)\\
& Tr(\bW)=q,\nonumber\\
& 0 \preceq \bW \preceq I, \nonumber\\
&\bK(i,i')=\EE[\phi_i(X_i)\phi_{i'}(X_{i'})],\quad 1\leq i,i'\leq p\nonumber\\
&\EE[\phi_i(X_i)^2]=1,\quad 1\leq i \leq p\nonumber\\
&\EE[\phi_i(X_i)]=0,\quad 1\leq i \leq p.\nonumber
\end{align}
Then $\bK^*$ is an optimal solution of optimization \eqref{opt:MCPCA-main} and $\sum_{r=1}^{q}\lambda_r(\bK^*)=Tr(\bW^* \bK^*)$.
\end{lemma}
\begin{proof}
The proof follows from the fact that the $q$ Ky Fan norm of a matrix $\bK$ is the solution of the following optimization \cite{boyd2004convex}:
\begin{align}
\max_{\bW}\quad &Tr(\bW \bK)\\
& Tr(\bW)=q,\nonumber\\
& 0 \preceq \bW \preceq I. \nonumber
\end{align}
\end{proof}
Consider the trace formulation of optimizations \eqref{opt:MCPCA-main} and \eqref{opt:MC-PCA-variance} according to Lemma \ref{thm:MC-PCA-trace}:

\begin{subequations}
\begin{align}
\max_{\bW, \{\phi_i\}_{i=1}^{p}}\quad &\sum_{i,i'} w_{i,i'} \EE[\phi_i(X_i)\phi_{i'}(X_{i'})]\label{opt:trace1}\\
& Tr(\bW)=q,\nonumber\\
& 0 \preceq \bW \preceq I\nonumber\\
&\EE[\phi_i(X_i)^2]=1,\quad 1\leq i \leq p\nonumber\\
&\EE[\phi_i(X_i)]=0,\quad 1\leq i \leq p,\nonumber\\\nonumber\\
\max_{\bW, \{\phi_i\}_{i=1}^{p}}\quad &\sum_{i,i'} w_{i,i'} \EE\left[\left(\phi_i(X_i)-\bar{\phi_i}(X_i)\right)\left(\phi_{i'}(X_{i'}-\bar{\phi_{i'}}(X_{i'})\right)\right]\label{opt:trace2}\\
& Tr(\bW)=q,\nonumber\\
& 0 \preceq \bW \preceq I\nonumber\\
&var(\phi_i(X_i))=1,\quad 1\leq i \leq p.\nonumber
\end{align}
\end{subequations}
Let $\phi_i^*$ and $\bW^*$ be an optimal solution of \eqref{opt:trace1}. The set of functions $\{\phi_i^*\}_{i=1}^{p}$ and $\bW^*$ is feasible for optimization \eqref{opt:trace2}. Thus, $\varrho_1\leq \varrho_2$. Moreover, let $\phi_i^{**}$ and $\bW^*$ be an optimal solution of optimization \eqref{opt:trace2}. Let $\tilde{\phi}_i= \phi_i^{**}- \bar{\phi_i}^{**}$. The set of functions $\{\tilde{\phi}_i\}_{i=1}^{p}$ and $\bW^*$ is feasible for optimization \eqref{opt:trace1}. Thus, we have $\varrho_1\leq \varrho_2$. Therefore, we have that $\varrho_1= \varrho_2$. This completes the proof.
\end{proof}

\section{MCPCA for Jointly Gaussian Random Variables}\label{sec:gaussian}
\subsection{Problem Formulation}\label{subsec:formulation-gauss}
Let $(X_1,\cdots, X_p)$ be zero mean unit variance jointly Gaussian random variables with the covariance matrix $\bK_X$. Thus $\bK_X(i,i')=\rho_{i,i'}$ where $\rho_{i,i'}$ is the correlation coefficient between variables $X_i$ and $X_{i'}$. Let $|\rho_{i,i'}|< 1$ for $i\neq i'$. A sign vector $\bs=(s_1,s_2,\cdots,s_p)^T$ is a vector in $\mathbb{R}^{p}$ where $s_i\in\{-1,1\}$ for $1\leq i\leq p$.

Let $h_{j}(.)$ be the $j$-th Hermite-Chebyshev polynomial for $j\geq 0$. These polynomials form an orthonormal basis with respect to the Gaussian distribution \cite{lancaster1957some}:
\begin{align}
\mathbb{E}[h_{j}(X_i)\ h_{j'}(X_{i'})]= (\rho_{i,i'})^{j} \mathbf{1}\{j=j'\}.\nonumber
\end{align}
Moreover, because Hermite-Chebyshev polynomials have zero means over a Gaussian distribution we have
\begin{align}
\mathbb{E}[h_{j}(X_i)]= \mathbf{1}\{j=0\}, ~~ 1\leq i\leq p.
\end{align}
Using a basis expansion approach similar to \cite{feizi2015network} we have
\begin{align}\label{eq:phi-decompos-gauss}
\phi_i(X_i) = \sum_{j=1}^{\infty} a_{i,j}\ h_{j} (X_i),
\end{align}
where $\ba_{i}=(a_{i,1},a_{i,2},\dots)^T$ is the vector of projection coefficients. The constraint $\EE[\phi_i(X_i)^2]=1$ translates to $||\ba_{i}|| =1$ while the constraint $\EE[\phi_i(X_i)]=0$ is simplified to $a_{i,0}=0$ for $1\leq i \leq p$. We also have
\begin{align}\label{eq:gauss-hilbert-covariance}
\bK_{(\phi_1(X_1),...,\phi_p(X_p))}(i,i')= \sum_{j=1}^{\infty} a_{i,j} a_{i',j} (\rho_{i,i'})^j.
\end{align}
Thus the MCPCA optimization \eqref{opt:MCPCA-main} can be re-written as follows:
\begin{align}\label{opt:MCPCA-gauss}
\max_{\bK}\quad & \sum_{r=1}^{q} \lambda_r(\bK)\\
& \bK(i,i')=\sum_{j=1}^{\infty} a_{i,j} a_{i',j} (\rho_{i,i'})^j, ~~ 1\leq i,i'\leq p,\nonumber\\
& \|\ba_i\|_{2}=1, ~~ 1\leq i \leq p.\nonumber
\end{align}
Since $|\rho_{i,i'}|<1$ for $i\neq i'$, $(\rho_{i,i'})^j\to 0$ as $j\to\infty$. Thus we can approximate optimization \eqref{opt:MCPCA-gauss} with the following optimization
\begin{align}\label{opt:MCPCA-gauss-apx}
\max_{\bK}\quad & \sum_{r=1}^{q} \lambda_r(\bK)\\
& \bK(i,i')=\sum_{j=1}^{N} a_{i,j} a_{i',j} (\rho_{i,i'})^j, ~~ 1\leq i,i'\leq p,\nonumber\\
& \|\ba_i\|_{2}=1, ~~ 1\leq i \leq p,\nonumber
\end{align}
for sufficiently large $N$.
\begin{lemma}
Let $\rho_q^*$ and $\tilde{\rho}_q^*$ be optimal values of optimizations \eqref{opt:MCPCA-gauss} and \eqref{opt:MCPCA-gauss-apx}, respectively. For a given $\eps>0$, there exists $N_0$ such that if $N>N_0$ we have $|\rho_q^*-\tilde{\rho}_q^*|<\eps$.
\end{lemma}
\begin{proof}
The proof follows from the fact that the Ky Fan norm of a matrix is a continuous function of its elements and also $(\rho_{i,i'})^j\to 0$ as $j\to\infty$.
\end{proof}
For the bivariate case ($p=2$), the MCPCA optimization simplifies to the maximum correlation optimization \eqref{eq:mc}. For jointly Gaussian variables the maximum correlation optimization \eqref{eq:mc} results in global optimizers $\phi_i^*(X_i)=s_i X_i$ for $i=1,2$ \cite{lancaster1957some}. Sign variables $s_i$'s are chosen so that the correlation between $s_1 X_1$ and $s_2 X_2$ is positive. This can be immediately seen from the formulation \eqref{opt:MCPCA-gauss-apx} as well: maximizing the off-diagonal entry of a $2\times 2$ covariance matrix maximizes its top eigenvalue. For the bivariate case the global optimizer of optimization \eqref{opt:MCPCA-gauss} is $a_i^*=(\pm 1,0,0,...)$ for $i=1,2$ since $|\rho_{1,2}|> (\rho_{1,2})^j$ for $j\geq 2$. Using \eqref{eq:phi-decompos-gauss} and since $h_1(.)$ is the identity function, we obtain $\phi_i^*(X_i)=s_i X_i$ for $i=1,2$.

Let $\cK_{ext}$ be the set of covariance matrices of variables $s_i X_i$ where $s_i=\pm 1$ for $1\leq i\leq p$. In the bivariate case we have
\begin{align}
\cK_{ext}=\left\{\left( \begin{array}{cc}
1 & \rho_{1,2} \\
\rho_{1,2} & 1\end{array} \right),\left( \begin{array}{cc}
1 & -\rho_{1,2} \\
-\rho_{1,2} & 1\end{array} \right)\right\}.
\end{align}
Note that covariance matrices in $\cK_{ext}$ have similar eigenvalues. Moreover in the bivariate case every covariance matrix $\bK_{(\phi_1(X_1),\phi_2(X_2))}$ can be written as a convex combination of covariance matrices in $\cK_{ext}$. Thus, it is majorized by covariance matrices in $\cK_{ext}$ (Lemma \ref{lem:unitary-eq}). However in the multivariate case we may have covariance matrices that are not in the convex hull of $\cK_{ext}$. To illustrate this, let $p=3$ and consider
\begin{align}
\cK_{ext}=\left\{\left( \begin{array}{ccc}
1 & 0.9&0.9 \\
0.9 & 1&0.7  \\
0.9 & 0.7&1\end{array} \right),\left( \begin{array}{ccc}
1 & -0.9&-0.9 \\
-0.9 & 1&0.7  \\
-0.9 & 0.7&1\end{array} \right),\left( \begin{array}{ccc}
1 & -0.9&0.9 \\
-0.9 & 1&-0.7  \\
0.9 & -0.7&1\end{array} \right),\left( \begin{array}{ccc}
1 & 0.9&-0.9 \\
0.9 & 1&-0.7  \\
-0.9 & -0.7&1\end{array} \right)\right\}.
\end{align}
One can show that the covariance matrix
\begin{align}
\left( \begin{array}{ccc}
1 & 0.9^2&0.9^2 \\
0.9^2 & 1&0.7^2  \\
0.9^2 & 0.7^2&1\end{array} \right)
\end{align}
is not included in the convex hull of covariance matrices in $\cK_{ext}$. This covariance matrix results from having $a_{i,2}=1$ for $1\leq i\leq 3$. Thus techniques used to characterize global optimizers of the bivariate case may not extend to the multivariate case.
\subsection{Global MCPCA Optimizers}
Here we characterize global optimizers of optimization \eqref{opt:MCPCA-gauss}. Our main result is as follows:

\begin{theorem}\label{thm:gauss-major}
$\bK_X$ majorizes every $\bK\in\cK_X$.
\end{theorem}
This Theorem along with Lemma \ref{prop:connection-major} results in the following corollary.

\begin{corollary}
$\phi_i(X_i)=s_i X_i$ where $s_i\pm 1$ for $1\leq i\leq p$ provides a globally optimal solution for the MCPCA optimization \eqref{opt:MCPCA-gauss} for $1\leq q\leq p$.
\end{corollary}
Below we present the proof of Theorem \ref{thm:gauss-major}.

\begin{proof}
First we prove the following lemma:
\begin{lemma}\label{lem:A-decompos-diag}
Let $\bK$ be a $p\times p$ positive semidefinite matrix with unit diagonal elements. Let $\bK^{\odot j}$ be the $j$-th Hadamard power of $\bK$. Then there exist diagonal matrices $\bE_k$ for $1\leq k\leq p^{j-1}$ such that
\begin{align}
\bK^{\odot j}=\sum_{k=1}^{p^{j-1}} \bE_k \bK \bE_k,
\end{align}
where $\sum_{k} \bE_k^2=\bI$.
\end{lemma}
\begin{proof}
We prove this lemma for $j=2$. The case of $j>2$ can be shown by a successive application of the proof technique. Since $\bK$ is a positive semidefinite matrix we can write $\bK=\bU^T \bU$. Since diagonal elements of $\bK$ are one we have $\|\bu_i\|=1$ where $\bu_i=(u_{1,i},...,u_{p,i})^T$ is the $i$-th column of $\bU$. Then we have
\begin{align}
\bK^{\odot 2}=\sum_{k=1}^{p} \bE_k \bK \bE_k,
\end{align}
where
\begin{align}
\bE_k=\left( \begin{array}{cccc}
u_{k,1} & 0&\cdots&0  \\
0 & u_{k,2}&\cdots&0  \\
\vdots & \vdots&\ddots&\vdots  \\
0 & 0&\cdots&u_{k,n} \end{array} \right).
\end{align}
Moreover we have
\begin{align}
\sum_{k=1}^{p} \bE_k^2=\left( \begin{array}{cccc}
\|\bu_1\|^2 & 0&\cdots&0  \\
0 & \|\bu_2\|^2&\cdots&0  \\
\vdots & \vdots&\ddots&\vdots  \\
0 & 0&\cdots&\|\bu_p\|^2 \end{array} \right)=\bI.
\end{align}
\end{proof}
Next we prove the following result on matrix majorization:
\begin{lemma}\label{lem:major-A-B}
Let $\bK$ be a $p\times p$ positive semidefinite matrix with unit diagonal elements. Let
\begin{align}
\bX=\sum_{j=1}^{N} \bD_j \bK^{\odot j} \bD_j,
\end{align}
where $\bD_j$'s are diagonal matrices such that $\sum_{j=1}^{N}\bD_j^2=\bI$. Then $\lambda(\bK)\succ \lambda(\bX)$.
\end{lemma}
\begin{proof}
Using Lemma \ref{lem:A-decompos-diag} we can write
\begin{align}
\bX=\sum_{j=1}^{M} \bF_j \bK \bF_j,
\end{align}
where $M=(p^N-1)/(p-1)$ and $\sum_{j=1}^{M}\bF_j^2=\bI$. Then using Theorem 1 of \cite{bapat1985majorization} completes the proof.
\end{proof}
Let $\bK_{(\phi_1(X_1),...,\phi_p(X_p))}$ be the covariance matrix of transformed variables $\{\phi_i(X_i)\}_{i=1}^{p}$. Using \eqref{eq:phi-decompos-gauss} and for sufficiently large $N$ we have
\begin{align}
\bK_{(\phi_1(X_1),...,\phi_p(X_p))}=\sum_{j=1}^{N} \bA_j \bK_X^{\odot j} \bA_j,
\end{align}
where
\begin{align}
\bA_j=\left( \begin{array}{cccc}
a_{1,j} & 0&\cdots&0  \\
0 & a_{2,j}&\cdots&0  \\
\vdots & \vdots&\ddots&\vdots  \\
0 & 0&\cdots&a_{p,j} \end{array} \right).
\end{align}
Since $\|\ba_i\|^2=1$ we have $\sum_{j=1}^N \bA_j^2=\bI$. Using Lemma \ref{lem:major-A-B} completes the proof.
\end{proof}

\section{MCPCA for Finite Discrete Random Variables}\label{sec:discrete}
\subsection{Problem Formulation}\label{subsec:disc-problem-formulation}
Let $X_i$ be a discrete random variable with distribution $P_{X_i}$ over the alphabet $\cX_i=\{1, \dots , |\cX_i|\}$. Without loss of generality, we assume all alphabets have positive probabilities as otherwise they can be neglected, i.e., $P_{X_i}(x)>0$ for $x\in \{1, \dots , |\cX_i|\}$. Let $\phi_i(X_i):\cX_i\to \mathbb{R}$ be a function of random variable $X_i$ with zero mean and unit variance. Using a basis expansion approach similar to \cite{feizi2015network}, we have
\begin{align}\label{eq:dis-basis-expansion}
\phi_i(X_i)=\sum_{i=1}^{|\cX_i|} a_{i,j} \psi_{i,j}(X_i),
\end{align}
where
\begin{align}\label{eq:basis-discrete}
\psi_{i,j}(x)\triangleq\mathbf{1}\{x=j\} \frac{1}{\sqrt{P_{X_i}(x)}}.
\end{align}
Note that $\{\psi_{i,j}\}_{j=1}^{|\cX_i|}$ form an orthonormal basis with respect to the distribution of $X_i$ because
\begin{align}\label{eq:orth-normal-psi}
&\EE[\psi_{i,j}(X_i)^2]=1,\quad 1\leq j\leq |\cX_i|\\
&\EE[\psi_{i,j}(X_i)\psi_{i,j'}(X_i)]=0,\quad 1\leq j\neq j'\leq |\cX_i|.\nonumber
\end{align}
Moreover we have
\begin{align}\label{eq:mean-psi}
\EE[\psi_{i,j}(X_i)]=\sqrt{P_{X_i}(j)},\quad 1\leq j\leq |\cX_i|.
\end{align}
Let $P_{X_i,X_{i'}}$ be the joint distribution of discrete variables $X_i$ and $X_{i'}$. Define a matrix $\bQ_{i,i'}\in\mathbb{R}^{|\cX_i|\times |\cX_{i'}|}$ whose $(j,j')$ element is
\begin{align}\label{def:Q}
\bQ_{i,i'}(j,j')\triangleq \frac{P_{X_i,X_{i'}}(j,j')}{\sqrt{P_{X_i}(j)P_{X_{i'}}(j')}}.
\end{align}
This matrix is called the $Q$-matrix of the distribution $P_{X_i,X_{i'}}$. Note that
\begin{align}
\EE[\psi_{i,j}(X_i)\psi_{i',j'}(X_{i'})]=\bQ_{i,i'}(j,j').
\end{align}
For $i=1, \dots, p$, let
\begin{align}\label{def:ai-sqrtpi}
&\ba_i\triangleq \left( a_{i,1},a_{i,2},\ldots,a_{i,|\mathcal{X}_{i}| } \right) ^T \\
&\sqrt{\bp_i}\triangleq \left( \sqrt{P_{X_i}(1)},\sqrt{P_{X_i}(2)},\ldots,\sqrt{P_{X_i}(|\mathcal{X}_{i}| )}\right) ^T.\nonumber
\end{align}
\begin{theorem}\label{thm:mcpca-opt-ai}
Let $\{\ba_i^*\}_{i=1}^{p}$ be an optimal solution of the following optimization:
\begin{align}\label{opt:MC-PCA-ai}
\max_{\{\ba_i\}_{i=1}^{p}}\quad &\sum_{r=1}^{q} \lambda_{r}(\bK)\\
& \bK(i,i')=\ba_i^T \bQ_{i,i'} \ba_{i'}, ~~ 1\leq i,i'\leq p,\nonumber\\
& \|\ba_i\|_{2}=1, ~~ 1\leq i \leq p, \nonumber \\
& \ba_i \perp \sqrt{\bp_i}, ~~ 1\leq i \leq p. \nonumber
\end{align}
Then, $\bK^*$ is an optimal solution of MCPCA optimization \eqref{opt:MCPCA-main}.
\end{theorem}
\begin{proof}
Consider $\bK\in\cK_{X}$ in the feasible region of MCPCA optimization \eqref{opt:MCPCA-main}. We have
\begin{align}
\bK(i,i')=\EE[\phi_i(X_i)\phi_{i'}(X_{i'})],~~ 1\leq i,i'\leq p
\end{align}
where $\EE[\phi_i(X_i)]=0$, and $\EE[\phi_i(X_i)^2]=1$ for all $1\leq
i\leq p$. Using \eqref{eq:dis-basis-expansion}, we can represent these functions in terms of the basis functions:
\begin{align}
\phi_i(X_i)&=\sum_{j=1}^{|\cX_i|} a_{i,j} \psi_{i,j}(X_i),\\
\phi_{i'}(X_{i'})&=\sum_{j=1}^{|\cX_{i'}|} a_{i',j'} \psi_{i',j'}(X_{i'}).\nonumber
\end{align}
Using \eqref{eq:orth-normal-psi}, the constraint $\EE[\phi_i(X_i)^2]=1$ would be translated into $\|\ba_i\|_{2}$ for $1\leq i\leq p$. Moreover using \eqref{eq:mean-psi}, the constraint $\mathbb{E}[\phi_i(X_i)]=0$ is simplified to $\sum_{j=1}^{|\cX_i|} a_{i,j} \sqrt{P_{X_i}(j)}=0$ for $1\leq i\leq p$. We also have
\begin{align}
\mathbb{E}[\phi_i(X_i)\phi_{i'}(X_{i'})]= \sum_{j=1}^{|\cX_{i}|} \sum_{j'=1}^{|\cX_{i'}|} a_{i,j} a_{i',j'}\ \mathbb{E}[\psi_{i,j} (X_i) \psi_{i',j'} (X_{i'})]=\ba_i^T \bQ_{i,i'} \ba_{i'}.
\end{align}
This shows every feasible point of optimization \eqref{opt:MCPCA-main} corresponds to a feasible point of optimization \eqref{opt:MC-PCA-ai}. The inverse argument is similar. This completes the proof.
\end{proof}
Recall that $\sigma_j(\bQ_{i,i'})$ is the $j$-th largest singular value of the matrix $\bQ_{i,i'}$ corresponding to left and right singular vectors $\bu_j(\bQ_{i,i'})$ and $\bv_j(\bQ_{i,i'})$, respectively.
\begin{lemma}\label{lem:sing-value-Q}
$\sigma_{1}(\bQ_{i,i'})=1$, $\bu_{1}(\bQ_{i,i'})=\sqrt{\bp_i}$ and $\bv_{1}(\bQ_{i,i'})=\sqrt{\bp_{i'}}$.
\end{lemma}
\begin{proof}
First we show that the maximum singular value of the matrix $\bQ_{i,i'}$ is less than or equal to one. To show that, it is sufficient to show that for every vectors $\ba_1$ and $\ba_2$ such that $\|\ba_1\|=1$ and $\|\ba_2\|=1$, we have $\ba_1^T \bQ_{i,i'} \ba_2 \leq 1$. To show this, we define random variables $\Upsilon_1$ and $\Upsilon_2$ such that
\begin{align*}
\PP\big(\Upsilon_1=\frac{a_{1,j}}{\sqrt{P_{X_i}(j)}},\Upsilon_2=\frac{a_{2,j'}}{\sqrt{P_{X_{i'}}(j')}}\big)=P_{X_{i},X_{i'}}(j,j').
\end{align*}
Using Cauchy-Schwartz inequality, we have
\begin{align*}
\mathbf{a}_1^T \bQ_{i,i'} \mathbf{a}_2= \EE [\Upsilon_1 \Upsilon_2] \leq \sqrt{\EE[\Upsilon_1^2] \EE[\Upsilon_2^2]}=||\mathbf{a}_1|| \  ||\mathbf{a}_2||= 1.
\end{align*}
Therefore, the maximum singular value of $\bQ_{i,i'}$ is at most one.

Moreover $\sqrt{\bp_i}$ and $\sqrt{\bp_{i'}}$ are right and left singular vectors of the matrix $\bQ_{i,i'}$ corresponding to the singular value one because $\bQ_{i,i'} \sqrt{\bp_{i'}}=\sqrt{\bp_i}$ and $\sqrt{\bp_i}^T \bQ_{i,i'}=\sqrt{\bp_{i'}}^T$.
\end{proof}

In the following we use similar techniques to the ones employed in \cite{feizi2015network} to formulate an alternative and equivalent optimization to \eqref{opt:MC-PCA-ai} without orthogonality constraints which proves to be useful in characterizing a globally optimal MCPCA solution when $q=1$.

Consider the matrix $\tI_i\triangleq I_{|\mathcal{X}_i|} - \sqrt{\mathbf{p}_i} \sqrt{\mathbf{p}_i}^T $. This matrix is positive semidefinite and the only vectors in its null space are $\mathbf{0}$ and $\sqrt{\mathbf{p}_i}$. This is because for any vector $\bx$ we have
\begin{align}
\mathbf{x}^T \left( I_{|\mathcal{X}_i|} - \sqrt{\mathbf{p}_i} \sqrt{\mathbf{p}_i}^T \right) \mathbf{x} = ||\bx||_2^2- (\mathbf{x} \sqrt{\mathbf{p}_i})^2 \ge 0,
\end{align}
where the Cauchy-Schwartz inequality and $||\sqrt{\mathbf{p}_i}||_2^2=1$ are used. The inequality becomes an equality if and only if $\mathbf{x}=0$ or $\mathbf{x}= \sqrt{\mathbf{p}_i}$. Moreover we have $\lambda_j(\tI_i)=1$ for $1\leq j<|\cX_i|$ because
\begin{align}
\left(I_{|\cX_i|} - \sqrt{\bp_i}\sqrt{\bp_i}^T \right) \bu_j(\tI_i)= \bu_j(\tI_i)- \sqrt{\bp_i}\sqrt{\bp_i}^T \bu_j(\tI_i)=\bu_j(\tI_i),
\end{align}
where the last equality follows from the fact that $\bu_j(\tI_i)$ is orthogonal to $\bu_{|\cX_i|}(\tI_i)=\sqrt{\bp_i}$.

Define $\bA_i\in \mathbb{R}^{|\cX_i|\times |\cX_i|}$ as follows:
\begin{align}\label{eq:Ai-inverseBi}
\bA_i\triangleq \Bigg( \left[\bu_1(\tI_i), \dots, \bu_{|\cX_i|-1}(\tI_i) \right] \left[ \bu_1(\tI_i), \dots, \bu_{|\cX_i|-1}(\tI_i) \right]^T\Bigg).
\end{align}

\begin{theorem}\label{thm:mcpca-opt-bi}
Let $\{\bb_i^*\}_{i=1}^{p}$ be an optimal solution of the following optimization:
\begin{align}\label{opt:MC-PCA-bi}
\max_{\{\bb_i\}_{i=1}^{p}}\quad &\sum_{r=1}^{q} \lambda_{r}(\bK)\\
& \bK(i,i')=\bb_i^T \left(\bQ_{i,i'}-\sqrt{\bp_i}\sqrt{\bp_{i'}}^T\right) \bb_{i'}, ~~ 1\leq i,i'\leq p,\nonumber\\
& \|\bb_i\|_{2}=1, ~~ 1\leq i \leq p. \nonumber
\end{align}
Then, $\{\ba_i^*\}_{i=1}^{p}$ is an optimal solution of optimization \eqref{opt:MC-PCA-ai} where $\ba_i^*=\bA_i \bb_i^*$.
\end{theorem}
\begin{proof}
We consider unit variance formulation of the MCPCA optimization \eqref{opt:MC-PCA-variance}. We have
\begin{align*}
& \mathbb{E}[(\phi_i(X_i)- \bar{\phi_i}(X_i))(\phi_{i'}(X_{i'})- \bar{\phi_{i'}}(X_{i'}))]= \mathbb{E}[\phi_i(X_i) \phi_{i'}(X_{i'})] - \bar{\phi_i}(X_i) \bar{\phi_{i'}}(X_{i'}) \nonumber \\
& = \mathbf{a}_i^T \bQ_{i,i'} \mathbf{a}_{i'}- (\mathbf{a}_i^T \sqrt{\mathbf{p}_i}) (\mathbf{a}_{i'}^T \sqrt{\mathbf{p}_{i'}})= \mathbf{a}_i^T \left( \bQ_{i,i'}- \sqrt{\mathbf{p}_i} \sqrt{\mathbf{p}_{i'}}^T \right) \mathbf{a}_{i'}.\nonumber
\end{align*}
Moreover we have
\begin{align}
var(\phi_i(X_i))=\mathbb{E}[\phi_i(X_i)^2]- (\mathbb{E}[\phi_i(X_i)])^2= ||\mathbf{a}_i||_2^2- (\mathbf{a}_i^T \sqrt{\mathbf{p}_i})^2= \mathbf{a}_i^T \left( I - \sqrt{\mathbf{p}_i} \sqrt{\mathbf{p}_i}^T \right) \mathbf{a}_i.\nonumber
\end{align}
Therefore optimization \eqref{opt:MC-PCA-variance} can be written as
\begin{align}\label{opt:MC-PCA-variance2}
\max_{\{\ba_i\}_{i=1}^{p}}\quad &\sum_{r=1}^{q} \lambda_{r}(\bK)\\
&\bK(i,i')=\ba_i^T \left( \bQ_{i,i'}- \sqrt{\mathbf{p}_i} \sqrt{\mathbf{p}_{i'}}^T\right)\ba_{i'},\quad 1\leq i,i'\leq p\nonumber\\
&\ba_i^T \left( I - \sqrt{\mathbf{p}_i} \sqrt{\mathbf{p}_i}^T \right) \ba_i=1,\quad 1\leq i \leq p\nonumber.
\end{align}
We can write $I - \sqrt{\mathbf{p}_i} \sqrt{\mathbf{p}_i}^T= \bB_i \bB_i^T$ (since $ I - \sqrt{\mathbf{p}_i} \sqrt{\mathbf{p}_i}^T $ is positive semidefinte) where
\begin{align}\label{eq:Bi}
\bB_i \triangleq \sqrt{I_{|\mathcal{X}_i|}- \sqrt{\mathbf{p}_i} \sqrt{\mathbf{p}_i}^T}.
\end{align}

Define $\mathbf{b}_i\triangleq \bB_i \mathbf{a}_i$. Thus, $\ba_i^T \left( I - \sqrt{\mathbf{p}_i} \sqrt{\mathbf{p}_i}^T \right) \ba_i=1$ can be written as $\mathbf{b}_i^T \mathbf{b}_i=||\mathbf{b}_i||_2^2=1$. The vector $\sqrt{\mathbf{p}_i}$ is the eigenvector corresponding to eigenvalue zero of the matrix $\bB_i$ ($\lambda_{|\cX_i|}(\bB_i)=0$). Other eigenvalues of $\bB_i$ is equal to one. Since $\bB_i$ is not invertible, there are many choices for $\mathbf{a}_i$ as a function of $\mathbf{b}_i$.
\begin{align}\label{eq:ba-i-choices}
\mathbf{a}_i = \big( [\bu_1(\bB_i), \dots, \bu_{|\mathcal{X}_i|-1}(\bB_i)] [\bu_1(\bB_i), \dots, \bu_{|\mathcal{X}_i|-1}(\bB_i)]^T \big)\mathbf{b}_i+ \alpha_i \sqrt{\mathbf{p}_i}= \bA_i \mathbf{b}_i + \alpha_i  \sqrt{\mathbf{p}_i},
\end{align}
where $\alpha_i$ can be an arbitrary scalar (note that $\bu_i(\tI_i)=\bu_i(\bB_i)$). However since the desired $\ba_i$ of optimization \eqref{opt:MC-PCA-ai} is orthogonal to the vector $\sqrt{\mathbf{p}_i}$, we choose $\alpha_i=0$ (i.e., according to Lemma \ref{lem:MC-PCA-variance}, in order to obtain a mean zero solution of the MCPCA optimization \eqref{opt:MCPCA-main}, we subtract the mean from the optimal solution of optimization \eqref{opt:MC-PCA-variance}.)
Therefore we have
\begin{align}\label{eq:zii'-bi-bi'}
\mathbf{a}_i^T \left( \bQ_{i,i'}- \sqrt{\mathbf{p}_i} \sqrt{\mathbf{p}_{i'}}^T \right) \mathbf{a}_{i'} = \mathbf{b}_i^T \bA_i^T \left(\bQ_{i,i'}- \sqrt{\mathbf{p}_i} \sqrt{\mathbf{p}_{i'}}^T \right)\bA_{i'} \mathbf{b}_{i'}.
\end{align}
Moreover using Lemma \ref{lem:sing-value-Q}, we have
\begin{align}
\bQ_{i,i'}=\sqrt{\bp_i}\sqrt{\bp_{i'}}^T+\sum_{j\geq 2} \sigma_j(\bQ_{i,i'}) \bu_j(\bQ_{i,i'}) \bu_j(\bQ_{i',i})^T.
\end{align}
Thus,
\begin{align}\label{eq:simplify-AiQii'}
\bA_i^T\left(\bQ_{i,i'}-\sqrt{\bp_i}\sqrt{\bp_{i'}}\right) \bA_{i'}^T&=\bA_i^T\left(\sum_{j\geq 2} \sigma_j(\bQ_{i,i'}) \bu_j(\bQ_{i,i'}) \bu_j(\bQ_{i',i})^T \right)\bA_{i'}^T\\
&=\sum_{j\geq 2} \sigma_j(\bQ_{i,i'}) \left(\bA_i^T \bu_j(\bQ_{i,i'})\right) \left(\bA_{i'}^T \bu_j(\bQ_{i',i})\right)^T\nonumber\\
&\stackrel{(I)}{=}\sum_{j\geq 2} \sigma_j(\bQ_{i,i'}) \left(\bu_j(\bQ_{i,i'})\right) \left(\bu_j(\bQ_{i',i})\right)^T\nonumber\\
&=\bQ_{i,i'}-\sqrt{\bp_i}\sqrt{\bp_{i'}},\nonumber
\end{align}
where equality (I) comes from expanding $\bu_j(\bQ_{i,i'})$ over the basis $\{\bu_{k}(\tI_i)\}_{k=1}^{|\cX_i|-1}$ and the fact that $\bu_j(\bQ_{i,i'})\perp \sqrt{\bp_i}$ for $j\geq 2$. Using equation \eqref{eq:simplify-AiQii'} in \eqref{eq:zii'-bi-bi'} completes the proof.
\end{proof}

\subsection{A Globally Optimal MCPCA Solution for the Rank One Constraint}\label{subsec:q=1}
In this part first we characterize an upper bound for the objective value of optimization \eqref{opt:MC-PCA-B} for $1\leq q\leq p$. Then, we construct a solution that achieves this upper bound for $q=1$.

Define a matrix $\bB\in\mathbb{R}^{p \times \sum_{i=1}^{p} |\cX_i|}$ such that
\begin{align}\label{eq:def-B}
\bB\triangleq\left( \begin{array}{cccc}
\bb_1 & 0&\cdots&0  \\
0 & \bb_2&\cdots&0  \\
\vdots & \vdots&\ddots&\vdots  \\
0 & 0&\cdots&\bb_p \end{array} \right).
\end{align}
Optimization \eqref{opt:MC-PCA-bi} can be written as
\begin{align}\label{opt:MC-PCA-B}
\max_{\{\bb_i\}_{i=1}^{p},\{\bv_r\}_{r=1}^{q}}\quad &\sum_{r=1}^{q} \bv_r^T \bB^T \bR \bB \bv_r\\
& \bB^T \bB=I, \nonumber\\
& \bv_r^T \bv_r=1, \quad 1\leq r\leq q,\nonumber\\
& \bv_r^T \bv_s=0, \quad 1\leq r\neq s\leq q,\nonumber
\end{align}
where $\bB$ has the structure defined in \eqref{eq:def-B}, and $\bR\in\mathbb{R}^{\sum_{i=1}^{p} |\cX_i|\times \sum_{i=1}^{p} |\cX_i|}$ where
\begin{align}\label{eq:def-Rii'}
\bR_{i,i'}= \left(\bQ_{i,i'}-\sqrt{\bp_i}\sqrt{\bp_{i'}}^T\right).
\end{align}
\begin{lemma}\label{prop:upper-bound}
The optimal value of optimization \eqref{opt:MC-PCA-B} is upper bounded by $\sum_{r=1}^{q}\lambda_r(\bR)$.
\end{lemma}
\begin{proof}
Define $\bw_r\triangleq \bB \bv_r$. We have
\begin{align}
\bw_r^T \bw_r&=\bv_r^T \bB^T \bB \bv_r=\bv_r^T \bv_r=1,\quad 1\leq r\leq q\\
\bw_r^T \bw_s&=\bv_r^T \bB^T \bB \bv_s=0,\quad 1\leq r\neq s\leq q.\nonumber
\end{align}
Thus,
\begin{align}\label{opt:relax-MC-PCA-B}
\max_{\{\bw_r\}_{r=1}^{q}}\quad &\sum_{r=1}^{q} \bw_r^T \bR \bw_r\\
& \bw_r^T \bw_r=1, \quad 1\leq r\leq q\nonumber\\
& \bw_r^T \bw_s=0, \quad 1\leq r\neq s\leq q\nonumber
\end{align}
is a relaxation of optimization \eqref{opt:MC-PCA-B}. The optimal solution of this optimization is achieved when $\bw_r=\bu_r(\bR)$ for $1\leq r\leq q$. This completes the proof.
\end{proof}

\begin{theorem}\label{thm:q=1}
Let $\bu_1(\bR)=(\bu_{1,1}(\bR),\bu_{1,2}(\bR),\cdots,\bu_{1,p}(\bR))^T$ where $\bu_{1,i}\in\mathbb{R}^{|\cX_i|}$. Then,
\begin{align}\label{eq:bi-for-q=1}
\bb_i^*=\frac{\bu_{1,i}(\bR)}{\|\bu_{1,i}(\bR)\|},\quad 1\leq i\leq p,
\end{align}
is an optimal solution of optimization \eqref{opt:MC-PCA-bi} when $q=1$.
\end{theorem}
\begin{proof}
Let $\bv_1=(\bv_{1,1},\cdots,\bv_{1,p})^T$. Choosing $\bv_{1,i}=\|\bu_{1,i}(\bR)\|$ and $\bb_i^*$ according to \eqref{eq:bi-for-q=1} achieves the upper bound provided in Lemma \ref{prop:upper-bound} for the case of $q=1$. This completes the proof.
\end{proof}

\subsection{MCPCA Computation Using a Block Coordinate Descend Algorithm}\label{subsec:comp-MCPCA-dis}
Here we provide a block coordinate descend algorithm to solve the MCPCA optimization for finite discrete variables with a general distribution for an arbitrary $1\leq q\leq p$. We then show that the algorithm converges to a stationary point of the MCPCA optimization.

Let $\bv_r=(\bv_{r,1},\cdots,\bv_{r,p})^T$. optimization \eqref{opt:MC-PCA-ai} can be written as
\begin{align}\label{opt:MC-PCA-sum-all-blocks}
\max_{\{\ba_i\}_{i=1}^{p},\{\bv_r\}_{r=1}^{q}}\quad &\sum_{r=1}^{q} \sum_{i=1}^{p}\sum_{i'=1}^{p} v_{r,i} v_{r,i'} \ba_i^T \bQ_{i,i'}\ba_{i'}\\
& \|\ba_i\|=1,\quad 1\leq i\leq p, \nonumber\\
& \ba_i\perp \sqrt{\bp_i},\quad 1\leq i\leq p, \nonumber\\
& \bv_r^T \bv_r=1, \quad 1\leq r\leq q,\nonumber\\
& \bv_r^T \bv_s=0, \quad 1\leq r\neq s\leq q.\nonumber
\end{align}

\begin{lemma}\label{lem:update-ai}
Let
\begin{align}
\bw_k=\sum_{r=1}^{q}\sum_{i\in\{1,...,p\}-\{k\}} v_{r,k} v_{r,i} \bQ_{k,i} \ba_i.
\end{align}
If all variables except $\ba_k$ are fixed in the feasible set of optimization \eqref{opt:MC-PCA-sum-all-blocks}, then
\begin{align}\label{eq:optimal-ai-coordinate}
\ba_k^*=\frac{\bw_k-\sqrt{\bp_k}^T\bw_k \sqrt{\bp_k}}{\|\bw_k-\sqrt{\bp_k}^T\bw_k \sqrt{\bp_k}\|}
\end{align}
is an optimal solution of the constrained optimization \eqref{opt:MC-PCA-sum-all-blocks} if $\|\bw_k-\sqrt{\bp_k}^T\bw_k \sqrt{\bp_k}\|\neq 0$. If $\|\bw_k-\sqrt{\bp_k}^T\bw_k \sqrt{\bp_k}\|=0$, every unit norm vector $\ba_k^*\perp \bw_k$ is an optimal solution of the constrained optimization \eqref{opt:MC-PCA-sum-all-blocks}.
\end{lemma}
\begin{proof}
Under the condition of Lemma \ref{lem:update-ai}, optimization \eqref{opt:MC-PCA-sum-all-blocks} is simplified to the following optimization:
\begin{align}
\max_{\ba_k}\quad& \ba_k^T \bw_k \\
& \|\ba_k\|=1,\nonumber\\
& \ba_k \perp \sqrt{\bp_k}.\nonumber
\end{align}
Writing $\bw_k=\sqrt{\bp_k}^T\bw_k\sqrt{\bp_k}+\left(\bw_k-\sqrt{\bp_k}^T\bw_k\sqrt{\bp_k}\right)$, we have
\begin{align}
\ba_k^T \bw_k=\ba_k^T \left(\bw_k-\sqrt{\bp_k}^T\bw_k\sqrt{\bp_k}\right),
\end{align}
since $\ba_k \perp \sqrt{\bp_k}$. This completes the proof.
\end{proof}

\begin{lemma}\label{lem:update-vr}
If all variables except $\{\bv_r\}_{r=1}^{q}$ are fixed in the feasible set of optimization \eqref{opt:MC-PCA-sum-all-blocks}, then $\bv_r^*=\bu_{r}(\bK)$ where $\bK(i,i')=\ba_{i}^T \bQ_{i,i'}\ba_{i'}$.
\end{lemma}
\begin{proof}
The proof follows from the eigen decomposition of the covariance matrix $\bK$.
\end{proof}

We use Lemmas \ref{lem:update-ai} and \ref{lem:update-vr} to propose a block coordinate descend Algorithm \ref{alg:MCPCA} to compute MCPCA.

\begin{algorithm}[t]
\caption{A Block Coordinate Descend Algorithm to Compute MCPCA for Finite Discrete Variables}
\begin{algorithmic}\label{alg:MCPCA}
\STATE \textbf{Input:} $P_{X_i,X_{i'}}$ for $1\leq i,i'\leq p$, $q$
\STATE \textbf{Initialization:} $\{\ba_i^{(0)}\}_{i=1}^{p}$ and $\{\bv_r^{(0)}\}_{r=1}^{q}$
\STATE \textbf{for} $j=0,1, \dots$
\STATE \hspace{.2 in} \textbf{for} $k=1,..., p$
\STATE \hspace{.4 in} \textbf{compute:} $\bw_k^{(j)}=\sum_{r=1}^{q}\sum_{i=1}^{k-1} v_{r,k}^{(j)} v_{r,i}^{(j)} \bQ_{k,i} \ba_i^{(j)}+\sum_{r=1}^{q}\sum_{i=k+1}^{p} v_{r,k}^{(j-1)} v_{r,i}^{(j-1)} \bQ_{k,i} \ba_i^{(j-1)}$
\STATE \hspace{.4 in} \textbf{update:} $\ba_k^{(j)}=\frac{\bw_k^{(j)}-\sqrt{\bp_k}^T\bw_k^{(j)} \sqrt{\bp_k}}{\|\bw_k^{(j)}-\sqrt{\bp_k}^T\bw_k^{(j)} \sqrt{\bp_k}\|}$, if $\|\bw_k^{(j)}-\sqrt{\bp_k}^T\bw_k^{(j)} \sqrt{\bp_k}\|\neq 0$
\STATE \hspace{.2 in} \textbf{compute:} $\bK^{(j)}$ where $\bK^{(j)}(i,i')=(\ba_{i}^{(j)})^T \bQ_{i,i'}\ba_{i'}^{(j)}$
\STATE \hspace{.2 in} \textbf{update:} $v_r^{(j)}=\bu_r(\bK^{(j)})$, for $1\leq r\leq q$
\STATE \hspace{.2 in} $\rho_q^{(j)}=\sum_{r=1}^{q}\lambda_r(\bK^{(j)})$
\STATE \textbf{end}
\end{algorithmic}
\end{algorithm}

\begin{theorem}\label{thm:greedy-convergance}
The sequence $\rho_q^{(j)}$ in Algorithm \ref{alg:MCPCA} is monotonically increasing and convergent. Moreover, if $\bK^{(j)}$ has top $q$ simple eigenvalues and $\|\bw_k^{(j)}-\sqrt{\bp_k}^T\bw_k^{(j)} \sqrt{\bp_k}\|\neq 0$ for $1\leq k\leq p$ and $j\geq 0$, then $\{\ba_i^{(j)}\}_{i=1}^{p}$ converges to stationary points of optimization \eqref{opt:MC-PCA-ai}.
\end{theorem}
\begin{proof}
According to Lemmas \ref{lem:update-ai} and \ref{lem:update-vr}, the sequence $\rho_q^{(j)}$ is increasing. Moreover, since it is bounded above (Theorem \ref{thm:prop-MCPCA}, part [i]), it is convergent. Moreover, under the conditions of Theorem \ref{thm:greedy-convergance}, at each step, Lemmas \ref{lem:update-ai} and \ref{lem:update-vr} provide a unique optimal solution for optimizing variables $\{\ba_i\}_{i=1}^{p}$ and $\{\bv_r\}_{r=1}^{q}$. Thus, $\{\ba_i^{(j)}\}_{i=1}^{p}$ converges to a stationary point of optimization \eqref{opt:MC-PCA-ai} (\cite{tseng2001convergence}).
\end{proof}

\section{Sample MCPCA}\label{sec:sample-MCPCA}
Principal component analysis is often applied to an observed data matrix whose rows and columns represent samples and features, respectively. In this part, first we review PCA and then formulate the sample MCPCA optimization (an MCPCA optimization computed over empirical distributions). We then study the consistency of sample MCPCA for both finite discrete and continuous variables.
\subsection{Review of PCA}\label{sec:review-PCA}
Let $\bX\in \mathbb{R}^{n\times p}$ be a data matrix:
\begin{align}\label{eq:data-matrix}
\bX=(\bx^{1},\dots, \bx^{p})= \left( \begin{array}{c}
\bx_1^T \\
\vdots \\
\bx_n^T \end{array} \right),
\end{align}
where $\bx_i$ and $\bx^{j}$ represent its $i$-th row and $j$-th column, respectively. Let $\bX(i,j)$, or interchangeably $\bX_{i,j}$, denote the $(i,j)$-th element of $X$. PCA aims to find orthonormal vectors $\bv_1,\dots,\bv_q$ where $\bv_i\in\mathbb{R}^p$ and $q\leq p$ such that the average mean squared error between $\hbx_k$ and $\bx_k$ for $1\leq k\leq n$ is minimized:
\begin{align}\label{opt:PCA}
\min_{\{\bv_i\}_{i=1}^{q},\btheta}\quad &\frac{1}{n}\sum_{k=1}^{n} \|\bx_k-\hbx_k\|^2\\
&\hbx_k=\btheta+\sum_{i=1}^{q} \left(\bv_i^T(\bx_k-\btheta)\right)\bv_i,\quad 1\leq k\leq n\nonumber\\
&\bv_i^T \bv_j=0,\quad 1\leq i\neq j\leq q\nonumber\\
&\bv_i^T \bv_i=1,\quad 1\leq i\leq q.\nonumber
\end{align}
Let
\begin{align}
&\btheta_n^*\triangleq\frac{1}{n}\sum_{i=1}^{n} \bx_i\\
&\bK_n\triangleq\frac{1}{n}\sum_{i=1}^{n} (\bx_i-\btheta^*)(\bx_i-\btheta^*)^T.\nonumber
\end{align}
$\bK_n$ and $\btheta_n^*$ are the empirical covariance matrix and the empirical mean of the data, respectively.
\begin{theorem}\label{thm:PCA}
$\btheta_n^*$ and $\bu_1(\bK_n)$,...,$\bu_q(\bK_n)$ provide an optimal solution for optimization \eqref{opt:PCA}.
\end{theorem}
\begin{proof}
See reference \cite{jolliffe2002principal}.
\end{proof}
By subtracting $\btheta_n^*$ from rows of the input matrix, the mean of each column becomes zero. This procedure is called {\it centring} the input data.
\subsection{Sample MCPCA for Finite Discrete Variables}\label{subsec:sample-MCPCA-disct}
Let $X_1$,..., $X_p$ be discrete variables with joint distribution $P_{X_1,...,X_p}$. Let the alphabet size of variables (i.e., $|\cX_i|$) be finite. We observe $n$ independent samples $\{\bx_i\}_{i=1}^{n}$ from this distribution. Let $\bX\in \mathbb{R}^{n\times p}$ be the data matrix \eqref{eq:data-matrix}. Sample MCPCA aims to find possibly nonlinear transformations of the data (i.e., $\phi_i(\bx^i)$ for $1\leq i\leq p$) to minimize the mean squared error (MSE) between the transformed data and its low rank approximation by $q$ orthonormal vectors $\bv_1$,...,$\bv_q$:

\begin{align}\label{opt:sample-MCPCA-data}
\min_{\{\bv_i\}_{i=1}^{q},\{\phi_i\}_{i=1}^{p}}\quad &\frac{1}{n}\sum_{k=1}^{n} \|\phi(\bx_k)-\hphi(\bx_k)\|^2\\
&\hphi(\bx_k)=\sum_{i=1}^{q} \left(\bv_i^T\phi(\bx_k)\right)\bv_i,\quad 1\leq k\leq n\nonumber\\
&\phi(\bx_k)= (\phi_1(\bX_{k,1}),\dots,\phi_p(\bX_{k,p})),\quad 1\leq k\leq n\nonumber\\
&\bv_i^T \bv_j=0,\quad 1\leq i\neq j\leq q,\nonumber\\
&\bv_i^T \bv_i=1,\quad 1\leq i\leq q,\nonumber\\
&\frac{1}{n}\sum_{k=1}^{n}\phi_i(\bX_{k,i})=0, \quad 1\leq i\leq p,\nonumber\\
&\frac{1}{n}\sum_{k=1}^{n}\phi_i(\bX_{k,i})^2=1, \quad 1\leq i\leq p.\nonumber
\end{align}
The constraint $\frac{1}{n}\sum_{k=1}^{n}\phi_i(\bX_{k,i})=0$ is similar to the centring step in the standard PCA where columns of the data matrix are transformed to have empirical zero means (Theorem \ref{thm:PCA}). The additional constraint $\frac{1}{n}\sum_{k=1}^{n}\phi_i(\bX_{k,i})^2=1$ makes columns of the transformed matrix to have equal norms.

Let $(Y_{1},\dots,Y_{p})$ be $p$ finite discrete random variables whose joint probability distribution $P_{Y_{1},\dots,Y_{p}}$ is equal to the empirical distribution of observed samples $\{\bx_i\}_{i=1}^{n}$. I.e.,
\begin{align}\label{eq:joint-emprical}
Pr(Y_1=j_1,...,Y_p=j_p)=\frac{1}{n}\sum_{k=1}^{n} \mathbf{1}\{\bX_{k,1}=j_1,...,\bX_{k,p}=j_p\}\
\end{align}
for $j_i\in\{1,2,...,|\cY_i|\}$.

\begin{theorem}\label{thm:sample-MCPCA}
Let $\bK^*$ be an optimal solution of the MCPCA optimization \eqref{opt:MCPCA-main} over variables $\{Y_i\}_{i=1}^{p}$ corresponding to transformation functions $\{\phi_i^*(Y_i)\}_{i=1}^{p}$. Then, $\{\bu_r(\bK^*)\}_{r=1}^{q},\{\phi_i^*(.)\}_{i=1}^{p}$ provide an optimal solution for optimization \eqref{opt:sample-MCPCA-data}.
\end{theorem}
\begin{proof}
Define $\bV\in\mathbb{R}^{p\times q}$ as follows:
\begin{align}\label{eq:W}
\bV=(\bv_1, \cdots, \bv_q).
\end{align}
Thus
\begin{align}\label{eq:phi-hat}
\hphi(\bx_k)=\bV \bV^T \phi(\bx_k),\quad 1\leq k\leq n.
\end{align}
We have
\begin{align}
\hphi(\bx_k)-\phi(\bx_k)=(I-\bV\bV^T) \phi(\bx_k).
\end{align}
Let $\bB_{\bV}\triangleq I-\bV\bV^T$. Note that $\bB_{\bV}=\bB_{\bV}^2=\bB_{\bV}^T$. Therefore we have
\begin{align}
&\frac{1}{n}\sum_{k=1}^{n} \|\phi(\bx_k)-\hphi(\bx_k)\|^2\\
&=\frac{1}{n} \sum_{k=1}^{n} \phi(\bx_k)^T \bB_{\bV}^T \bB_{\bV} \phi(\bx_k)\nonumber\\
&=\frac{1}{n} \sum_{k=1}^{n} \phi(\bx_k)^T (I-\bV\bV^T) \phi(\bx_k)\nonumber\\
&=Tr\left(\frac{1}{n}\sum_{k=1}^n \phi(\bx_k) \phi(\bx_k)^T\right)-Tr\left(\bV^T \left(\frac{1}{n}\sum_{k=1}^n \phi(\bx_k) \phi(\bx_k)^T\right) \bV\right)\nonumber\\
&= Tr(\bK)-Tr(\bV^T \bK \bV)= p-Tr(\bV^T \bK \bV),\nonumber
\end{align}
$\bK\triangleq \frac{1}{n}\sum_{k=1}^n \phi(\bx_k) \phi(\bx_k)^T$. Since $\{Y_i\}_{i=1}^{p}$ is distributed according to the empirical distribution of samples $\{\bx_k\}_{k=1}^{n}$, we have
\begin{align}
\bK(i,i')=\frac{1}{n}\sum_{k=1}^{n}\phi_i(\bX_{k,i})\phi_{i'}(\bX_{k,i'})=\EE[\phi_i(Y_i) \phi_{i'}(Y_{i'})].
\end{align}
Similarly the constraint $\frac{1}{n}\sum_{k=1}^{n}\phi_i(\bX_{k,i})=0$ is simplified to the constraint $\EE[\phi_i(Y_i)]=0$, while the constraint $\frac{1}{n}\sum_{k=1}^{n}\phi_i(\bX_{k,i})^2=1$ is translated to the constraint $\EE[\phi_i(Y_i)^2]=1$. Therefore, optimization \eqref{opt:sample-MCPCA-data} can be written as
\begin{align}\label{opt:MC-PCA-proof1}
\min_{\bV,\bK}\quad &p-Tr(\bV^T \bK \bV)\\
& \bV^T \bV=I_{q},\nonumber\\
& \bK\in\cK_{Y}.\nonumber
\end{align}
Moreover using \eqref{eq:W}, we have
\begin{align}
Tr(\bV^T \bK \bV)= \sum_{r=1}^{q} \bv_r^T \bK \bv_r.
\end{align}
Let $\lambda=\bv_r^T \bK \bv_r$. Since $\bv_r \lambda=\bK \bv_r$, $\bv_r$ is an eigenvector of $\bK$ corresponding to eigenvalue $\lambda_r(\bK)$. This simplifies optimization \eqref{opt:MC-PCA-proof1} to optimization \eqref{opt:MCPCA-main} and completes the proof.
\end{proof}

The following Theorem discusses the consistency of sample MCPCA for finite discrete variables.

\begin{theorem}\label{thm:const-discrete}
Let $\rho_q^*$ and $\tilde{\rho_q}^{(n)}$ be optimal MCPCA values over variables $\{X_i\}_{i=1}^{p}$ and $\{Y_i\}_{i=1}^{p}$. Let $p$ and $q$ be fixed. As $n\to \infty$, with probability one, $\tilde{\rho_q}^{(n)}\to\rho_q^*$.
\end{theorem}
\begin{proof}
The proof follows form the fact that for a fixed $p$ and $q$, as $n\to\infty$, eigenvalues of the empirical covariance matrix converge to the eigenvalues of the true covariance matrix, with probability one.
\end{proof}

\begin{algorithm}[t]
\caption{A Block Coordinate Descend Algorithm to Compute Sample MCPCA For Finite Discrete Variables}
\begin{algorithmic}\label{alg:sample-MCPCA}
\STATE \textbf{Input:} $\bX$, $q$
\STATE \textbf{Initialization:} $\{\phi_i^{(0)}\}_{i=1}^{p}$ and $\{\bv_r^{(0)}\}_{r=1}^{q}$
\STATE \textbf{for} $j=0,1, \dots$
\STATE \hspace{.2 in} \textbf{for} $k=1,..., p$
\STATE \hspace{.4 in} \textbf{compute:} $\bw_k^{(j)}=\sum_{r=1}^{q}\sum_{i=1}^{k-1} v_{r,k}^{(j)} v_{r,i}^{(j)} \phi_i^{(j)}(\bx^i)+\sum_{r=1}^{q}\sum_{i=k+1}^{p} v_{r,k}^{(j-1)} v_{r,i}^{(j-1)} \phi_i^{(j-1)}(\bx^i)$
\STATE \hspace{.4 in} \textbf{update:} $\phi_k^{(j)}=\EE[\bw_k^{(j)}|Y_k]/||\EE[\bw_k^{(j)}|Y_k]||$
, if $||\EE[\bw_k^{(j)}|Y_k]||\neq 0$
\STATE \hspace{.2 in} \textbf{compute:} $\bK^{(j)}$ where $\bK^{(j)}(i,i')=\frac{1}{n}\sum_{s=1}^{n}\phi_i^{(j)}(\bX_{s,i})\phi_{i'}^{(j)}(\bX_{s,i'})$
\STATE \hspace{.2 in} \textbf{update:} $v_r^{(j)}=\bu_r(\bK^{(j)})$, for $1\leq r\leq q$
\STATE \hspace{.2 in} $\rho_q^{(j)}=\sum_{r=1}^{q}\lambda_r(\bK^{(j)})$
\STATE \textbf{end}
\end{algorithmic}
\end{algorithm}

\subsection{Computation of Sample MCPCA for Finite Discrete Variables}\label{subsec:sample-MCPCA-comp-disc}
One way to compute sample MCPCA is to use empirical pairwise joint distributions in Algorithm \ref{alg:MCPCA}. However, forming and storing these empirical pairwise joint distributions may be expensive. Below, we discuss computation of the sample MCPCA optimization without forming pairwise joint distributions.

Let $\bv_r=(\bv_{r,1},\cdots,\bv_{r,p})^T$. The sample MCPCA optimization \eqref{opt:MCPCA-main} can be written as follows:
\begin{align}\label{opt:sample-MCPCA-trace-V}
\max_{\{\phi_i\}_{i=1}^{p},\{\bv_r\}_{r=1}^{q}}\quad &\sum_{r=1}^{q} \sum_{i=1}^{p}\sum_{i'=1}^{p} v_{r,i} v_{r,i'} \left(\frac{1}{n}\sum_{s=1}^{n}\phi_i(\bX_{s,i})\phi_{i'}(\bX_{s,i'})\right)\\
&\frac{1}{n}\sum_{s=1}^{n}\phi_i(\bX_{s,i})=0, \quad 1\leq i\leq p,\nonumber\\
&\frac{1}{n}\sum_{s=1}^{n}\phi_i(\bX_{s,i})^2=1, \quad 1\leq i\leq p\nonumber\\
& \bv_r^T \bv_r=1, \quad 1\leq r\leq q\nonumber\\
& \bv_r^T \bv_s=0, \quad 1\leq r\neq s\leq q.\nonumber
\end{align}

Let $(Y_{1},\dots,Y_{p})$ be $p$ finite discrete random variables whose joint probability distribution $P_{Y_{1},\dots,Y_{p}}$ is equal to the empirical distribution of observed samples $\{\bx_i\}_{i=1}^{n}$. Define the vector $\bw_k\in \mathbb{R}^n$ as follows:
\begin{align}\label{eq:wk}
\bw_k\triangleq\sum_{r=1}^{q}\sum_{i\in\{1,...,p\}-\{k\}} v_{r,k} v_{r,i} \phi_i(\bx^i).
\end{align}

\begin{lemma}\label{lem:update-phi}
If all variables except $\phi_k$ are fixed in the feasible set of optimization \eqref{opt:sample-MCPCA-trace-V}, then
\begin{align}\label{eq:optimal-ai-coordinate}
\phi_k^*(Y_k)=\EE[\bw_k|Y_k]/||\EE[\bw_k|Y_k]||,
\end{align}
is the optimal solution of the constrained optimization \eqref{opt:sample-MCPCA-trace-V} if $||\EE[\bw_k|Y_k]||\neq 0$. If $||\EE[\bw_k|Y_k]||=0$, every mean zero and unit norm $\phi_k^*$ is an optimal solution of the constrained optimization \eqref{opt:sample-MCPCA-trace-V}.
\end{lemma}
\begin{proof}
If all variables except $\phi_k$ are fixed, optimization \eqref{opt:sample-MCPCA-trace-V} can be simplified to
\begin{align}\label{opt:alaki1}
\max_{\phi_k}\quad& <\phi_k(\bx^k),\bw_k>\\
&\mathbf{1}^T \phi_k(\bx^k)=0,\nonumber\\
&||\phi_k(\bx^k)||^2=n.\nonumber
\end{align}
Note that since there exists $\phi_k$ such that $<\phi_k(\bx^k),\bw_k>\geq 0$, the constraint $||\phi_k(\bx^k)||^2=n$ can be replaced by the constraint $||\phi_k(\bx^k)||^2\leq n$. Now consider the following optimization:
\begin{align}\label{opt:alaki2}
\max_{\phi_k}\quad& <\phi_k(\bx^k),\bw_k>\\
&||\phi_k(\bx^k)||^2\leq n.\nonumber
\end{align}
We show that the optimal solution of optimization \eqref{opt:alaki2} has zero mean. For simplicity, we use $\phi_k$ instead of $\phi_k(\bx^k)$. We proceed by contradiction. Suppose $\phi_k^*$ is an optimal solution of optimization \eqref{opt:alaki2} whose mean is not zero (i.e., $\bar{\phi_k^*}\neq 0$). Consider the following solution:
\begin{align}
\tilde{\phi_k}=\sqrt{n}\frac{\phi_k^*-\bar{\phi_k^*}}{||\phi_k^*-\bar{\phi_k^*}||}.
\end{align}
Note that $||\tilde{\phi_k}||^2=n$. Thus, $\tilde{\phi_k}$ belongs to the feasible set of optimization \eqref{opt:alaki2}. Moreover we have
\begin{align}
&||\phi_k^*-\bar{\phi_k^*}||^2=||\phi_k^*-\frac{\mathbf{1}^T\phi_k^*}{n}||^2\\
&=||\phi_k^*||^2+(\frac{1}{n^2}-\frac{2}{n})(\mathbf{1}^T\phi_k^*)^2<||\phi_k^*||^2\leq n.\nonumber
\end{align}
Therefore,
\begin{align}\label{eq:alaki3}
\frac{\sqrt{n}}{||\phi_k^*-\bar{\phi_k^*}||}>1.
\end{align}
Using \eqref{eq:alaki3} and the fact that $\mathbf{1}^T \bw_k=0$, $\tilde{\phi_k}$ leads to a strictly larger objective value of optimization \eqref{opt:alaki2} than the one of $\phi_k^*$, which is a contradiction. Therefore, the optimal solution of optimization \eqref{opt:alaki2} has zero mean. Thus, optimization \eqref{opt:alaki2} is a tight relaxation of optimization \eqref{opt:alaki1}.

Define $\theta_k(Y_k)\triangleq \EE[\bw_k|Y_k]$. Thus, $<\phi_k(\bx^k),\bw_k>=\EE[\phi_k(Y_k)\theta_k(Y_k)]$. Moreover, $||\phi_k(\bx^k)||^2=n\EE[Y_k^2]$. Therefore, optimization \eqref{opt:alaki2} is simplified to the following optimization:
\begin{align}\label{opt:alaki4}
\max\quad& \EE[\phi_k(Y_k)\theta_k(Y_k)]\\
&\EE[Y_k^2]\leq 1\nonumber.
\end{align}
Using the Cauchy-Schwartz inequality completes the proof.
\end{proof}
To update variables $\{\bv_r\}_{r=1}^{q}$, one can use Lemma \ref{lem:update-vr}. Similarly to Algorithm \ref{alg:MCPCA}, to solve the sample MCPCA optimization for finite discrete variables, we propose Algorithm \ref{alg:sample-MCPCA} which is based on a block coordinate descend approach.

\begin{theorem}\label{thm:greedy-convergance-sample}
The sequence $\rho_q^{(j)}$ in Algorithm \ref{alg:sample-MCPCA} is monotonically increasing and convergent. Moreover, if $\bK^{(j)}$ has top $q$ simple eigenvalues and $||\EE[\bw_k^{(j)}|Y_k]||\neq 0$ for $1\leq k\leq p$ and $j\geq 0$, then $\{\phi_i^{(j)}\}_{i=1}^{p}$ converges to stationary points of optimization \eqref{opt:sample-MCPCA-trace-V}.
\end{theorem}
\begin{proof}
The proof is similar to the one of Theorem \ref{thm:greedy-convergance}.
\end{proof}

\begin{proposition}\label{prop:complexity}
Each iteration of Algorithm \ref{alg:sample-MCPCA} has a computational complexity of $\cO(p^3+np^2)$ and a memory complexity of $\cO(np)$.
\end{proposition}

\begin{remark}\label{remark:complexity}
\textup{The computational complexity of Isomap and LLE is $\cO(n^3)$ and $\cO(pn^2)$ while their memory complexity is $\cO(n^2)$ and $\cO(pn^2)$, respectively. Unlike Isomap and LLE, computational and memory complexity of MCPCA Algorithm \ref{alg:sample-MCPCA} scales linearly with the number of samples $n$ which makes it suitable for data sets with large number of samples.}
\end{remark}

\subsection{Sample MCPCA for Continuous Variables}\label{subsec:sample-MCPCA-cont}
In this part, we consider the case where $X_1$,..., $X_p$ are continuous variables with the density function $f_{X_1,...,X_p}$. Here we assume $X_1$,...,$X_p$ have bounded ranges. Without loss of generality, let $X_i\in[0,1]$ for $1\leq i\leq p$. Moreover, let the density function satisfy $f_{X_i}(x)>0$ for $x\in[0,1]$ and $1\leq i\leq p$. We observe $n$ independent samples $\{\bx_i\}_{i=1}^{n}$ from this distribution. The data matrix $X\in \mathbb{R}^{n\times p}$ is defined according to \eqref{eq:data-matrix}. Since $X_1$,...,$X_p$ are continuous, with probability one, each column of the matrix $X$ has $n$ distinct values. Thus, with probability one, there exists $\{\phi_i^*(.)\}$ such that $\phi_i^*(\bx^i)=\bw$ for $1\leq i\leq p$, where $\bw$ is a vector in $\mathbb{R}^n$ whose mean is zero and its norm is equal to $\sqrt{n}$. Therefore, with probability one, the optimal value of optimization \eqref{opt:sample-MCPCA-data} is equal to $p$.

In the continuous case, the space of feasible transformation functions has infinite {\it degrees of freedom}. Thus, by observing $n$ samples from these continuous variables, we over-fit functions to observed samples.
Note that in the case of having observations from finite discrete variables, transformation functions have finite degrees of freedom and if the number of samples are sufficiently large, over-fitting issue does not occur (Theorem \ref{thm:const-discrete}). One approach to overcome the over-fitting issue in the continuous case is to restrict the feasible set of optimization \eqref{opt:sample-MCPCA-data} to functions whose degrees of freedom are smaller than the number of observed samples $n$. One such family of functions is piecewise linear functions with $d$ degrees of freedom:

\begin{definition}\label{def:piecewise-linear}
Let $\bw\in\mathbb{R}^{d+1}$. $\cG_d(\bw)$ is defined as the set of all functions $g:[0,1]\to \mathbb{R}$ such that
\begin{align}\label{eq:piecewise-linear}
g_d(x) \triangleq
\left\{
	\begin{array}{ll}
		w_j  & \mbox{if } x=\frac{j}{d}, 0\leq j\leq d \\
		(w_{j+1}-w_{j})(Mx-j)+w_j & \mbox{if } \frac{j}{d}<x<\frac{j+1}{d}, 0\leq j\leq d-1
	\end{array}
\right.
\end{align}
Moreover, $\cG_d\triangleq \{\cG_d(\bw):\bw\in\mathbb{R}^{d+1}\}$.
\end{definition}

Let $\{\bx_k\}_{k=1}^{n}$ be observed sample from continuous variables $X_1$,...,$X_p$. Sample MCPCA aims to solve the following optimization:
\begin{align}\label{opt:sample-MCPCA-cont-data}
\min_{\{\bv_i\}_{i=1}^{q},\{\phi_i\}_{i=1}^{p}}\quad &\frac{1}{n}\sum_{k=1}^{n} \|\phi(\bx_k)-\hphi(\bx_k)\|^2\\
&\hphi(\bx_k)=\sum_{i=1}^{q} \left(\bv_i^T\phi(\bx_k)\right)\bv_i,\quad 1\leq k\leq n\nonumber\\
&\phi(\bx_k)= (\phi_1(\bX_{k,1}),\dots,\phi_p(\bX_{k,p})),\quad 1\leq k\leq n\nonumber\\
&\bv_i^T \bv_j=0,\quad 1\leq i\neq j\leq q\nonumber\\
&\bv_i^T \bv_i=1,\quad 1\leq i\leq q,\nonumber\\
&\frac{1}{n}\sum_{k=1}^{n}\phi_i(\bX_{k,i})=0, \quad 1\leq i\leq p,\nonumber\\
&\frac{1}{n}\sum_{k=1}^{n}\phi_i(\bX_{k,i})^2=1, \quad 1\leq i\leq p\nonumber\\
&\phi_i\in\cG_d,\quad 1\leq i\leq p. \nonumber
\end{align}

\begin{theorem}\label{thm:sample-MCPCA-cont}
Consider the following optimization:
\begin{align}\label{opt:sample-MCPCA-cont}
\max_{\{\phi_i\}_{i=1}^{p}}\quad &\sum_{r=1}^{q} \lambda_{r}(\bK)\\
&\bK(i,i')=\frac{1}{n}\sum_{k=1}^{n}\phi_i(\bX_{k,i})\phi_{i'}(\bX_{k,i'}),\quad 1\leq i,i'\leq p\nonumber\\
&\frac{1}{n}\sum_{k=1}^{n}\phi_i(\bX_{k,i})^2=1,\quad 1\leq i \leq p\nonumber\\
&\frac{1}{n}\sum_{k=1}^{n}\phi_i(\bX_{k,i})=0,\quad 1\leq i \leq p\nonumber\\
&\phi_i\in\cG_d,\quad 1\leq i\leq p. \nonumber
\end{align}
Let $\bK^*$ be an optimal solution of optimization \eqref{opt:sample-MCPCA-cont} corresponding to transformation functions $\{\phi_i^*\}_{i=1}^{p}$. Then, $\{\bu_r(\bK^*)\}_{r=1}^{q},\{\phi_i^*\}_{i=1}^{p}$ provide an optimal solution of optimization \eqref{opt:sample-MCPCA-cont-data}.
\end{theorem}
\begin{proof}
 The proof is similar to the one of Theorem \ref{thm:sample-MCPCA}.
\end{proof}

\begin{proposition}\label{prop:M=1-sampleMCPCA=PCA}
Let columns of the data matrix $\bX$ have zero means and unit variances. If $d=1$, the sample MCPCA optimization \eqref{opt:sample-MCPCA-cont-data} is equivalent to the PCA optimization \eqref{opt:PCA}.
\end{proposition}
\begin{proof}
For $d=1$, $\cG_{d}$ only contains linear functions. Since columns of the data matrix $\bX$ are assumed to be normalized, optimization \eqref{opt:sample-MCPCA-cont-data} is equivalent to optimization \eqref{opt:PCA}.
\end{proof}
\subsection{Computation of MCPCA and Sample MCPCA for Continuous Variables}\label{subsec:comp-cont}
Define discrete variables $Y_{i,d}$ whose alphabets are $\{1,2,..,d\}$ and
\begin{align}\label{eq:uniform-quant-vars}
Pr(Y_{1,d}=j_1,...,Y_{p,d}=j_p)=\int_{x_1=(j_1-1)/d}^{j_1/d}\dots \int_{x_p=(j_p-1)/d}^{j_p/d} f_{X_1,...,X_p}(x_1,...,x_p) dx_1...dx_p
\end{align}

Below we establish a connection between solutions of the MCPCA optimization over continuous variables and their discretized versions. We will use this connection to compute MCPCA and sample MCPCA over continuous variables.
\begin{theorem}\label{thm:computation-cont-dist}
Let $\rho_q^*$ and $\rhoh_{q,d}^*$ be optimal values of the MCPCA optimization \eqref{opt:MCPCA-main} over continuous variables $\{X_i\}_{i=1}^{p}$ and discrete variables $\{Y_{i,d}\}_{i=1}^{p}$, respectively. As $d\to\infty$, with probability one, $\rhoh_{q,d}^*\to\rho_q^*$. Moreover, let $\{\phih_{i,d}^*(.)\}$ be an optimal solution of the MCPCA optimization \eqref{opt:MCPCA-main} over discrete variables $\{Y_{i,d}\}_{i=1}^{p}$. Let $\bw_i=\left(\phih_{i,d}^*(1),\phih_{i,d}^*(1),...,\phih_{i,d}^*(d)\right)$. Then, as $d\to\infty$, with probability one, $\{g_d(\bw_i)\}$ is an optimal solution of the MCPCA optimization \eqref{opt:MCPCA-main} over continuous variables $\{X_{i}\}_{i=1}^{p}$.
\end{theorem}
\begin{proof}
For $1\leq i\leq p$, let $\phi_i:[0,1]\to\mathbb{R}$ be a feasible function in the MCPCA optimization \eqref{opt:MCPCA-main} over continuous variables $\{X_i\}_{i=1}^{p}$. Define $\phih_{i,d}:\{1,2,...,d\}\to\mathbb{R}$ such that
\begin{align}\label{eq:phihat}
\phih_{i,d}(j)\triangleq \phi_i((j-1)/d).
\end{align}
Below we show that as $d\to\infty$, with probability one, $\{\phih_{i,d}\}_{i=1}^{p}$ is feasible in the MCPCA optimization \eqref{opt:MCPCA-main} over discrete variables $\{Y_{i,d}\}_{i=1}^{p}$. We have
\begin{align}
\EE[\phih_{i,d}(Y_{i,d})]&=\sum_{j=1}^{d} Pr(Y_{i,d}=j) \phih_{i,d}(j)\\
&=\sum_{j=1}^{d} \int_{x=(j-1)/d}^{j/d} f_{X_i}(x) \phi_i(\frac{j-1}{d})\nonumber\\
&\to \sum_{j=1}^{d} \int_{x=(j-1)/d}^{j/d} \phi_i(x) f_{X_i}(x) dx\nonumber\\
&=\int_{x=0}^{1} \phi_i(x) f_{X_i}(x) dx=0.\nonumber
\end{align}
Similarly as $d\to\infty$, with probability one, $\EE[\phih_{i,d}(Y_{i,d})^2]=1$, and
\begin{align}
\EE[\phih_{i,d}(Y_{i,d}) \phih_{i',d}(Y_{i',d})]=\EE[\phi_i(X_i)\phi_{i'}(X_{i'})].
\end{align}

Now consider $\{\phih_{i,d}(Y_{i,d})\}_{i=1}^{p}$ as a feasible point for the MCPCA optimization \eqref{opt:MCPCA-main} over discrete variables $\{Y_{i,d}\}_{i=1}^{p}$.
For $1\leq i\leq p$, define
\begin{align}\label{eq:phi-bar}
\tilde{\phi_i}\triangleq g_d(\phih_{i,d}(1),\phih_{i,d}(1),\phih_{i,d}(2),...,\phih_{i,d}(d)).
\end{align}
Note that $\tilde{\phi_i}:[0,1]\to\mathbb{R}$. Similarly to the previous argument, as $d\to\infty$, with probability one, $\{\tilde{\phi_i}\}_{i=1}^{p}$ is a feasible point in the MCPCA optimization \eqref{opt:MCPCA-main} over continuous variables $\{X_i\}_{i=1}^{p}$. Moreover, as $d\to\infty$, with probability one, we have
\begin{align}
\EE[\tilde{\phi_i}(X_i)\tilde{\phi_{i'}}(X_{i'})]=\EE[\phih_{i,d}(Y_{i,d}) \phih_{i',d}(Y_{i',d})].
\end{align}

Consider $\{\phi_i^*\}_{i=1}^{p}$ as an optimal solution of optimization \eqref{opt:MCPCA-main} over continuous variables $\{X_i\}_{i=1}^{p}$ with the optimal value $\rho_q^*$. Construct $\{\phih_{i,d}(.)\}_{i=1}^{p}$ according to equation \eqref{eq:phihat}. As $d\to\infty$, with probability one, $\{\phih_{i,d}(.)\}_{i=1}^{p}$ is a feasible point for the MCPCA optimization \eqref{opt:MCPCA-main} over discrete variables $\{Y_{i,d}\}_{i=1}^{p}$ which leads to the MCPCA objective value $\rhoh_{q,d}=\rho_q^*$. Thus, $\rhoh_{q,d}^*\geq \rho_q^*$.

Now consider $\{\phih_{i,d}^*\}_{i=1}^{p}$ as an optimal solution of optimization \eqref{opt:MCPCA-main} over discrete variables $\{Y_{i,d}\}_{i=1}^{p}$ which leads to the MCPCA objective value $\rhoh_{q,d}^*$. Construct $\{\tilde{\phi_{i}}\}_{i=1}^{p}$ according to equation \eqref{eq:phi-bar}. As $d\to\infty$, with probability one, $\{\tilde{\phi_{i}}\}_{i=1}^{p}$ is a feasible point for the MCPCA optimization \eqref{opt:MCPCA-main} over continuous variables $\{X_i\}_{i=1}^{p}$ with the optimal value $\rho_q=\rho_{q,d}^*$. Thus, $\rho_q\geq \rho_{q,d}^*$. This completes the proof.
\end{proof}
Theorem \ref{thm:computation-cont-dist} simplifies the MCPCA computation over continuous variables $\{X_i\}_{i=1}^{p}$ to the MCPCA computation over discrete variables $\{Y_{i,d}\}_{i=1}^{d}$ which can be solved using Algorithm \ref{alg:MCPCA}. A similar approach can be taken to simplify the sample MCPCA optimization over continuous variables to the one of the discrete variables which can be solved using Algorithm \ref{alg:sample-MCPCA}.

Variable $Y_{i,d}$ provides a discretized version of the continuous variable $X_i$ where the position of knots (i.e., discretization thresholds) are uniformly spaced in the range of the variable. However the argument of Theorem \ref{thm:computation-cont-dist} can be extended to consider other nonuniform and data-dependent discretization as well. For example, in the case that we observe $n$ samples from $X_i$, one can choose the position of discretization knots to have equal number of samples in each discretization level. In the sample MCPCA implementation for continuous variables, we use such a nonuniform discretization approach.

\begin{figure}
\centering
  \includegraphics[width=\linewidth]{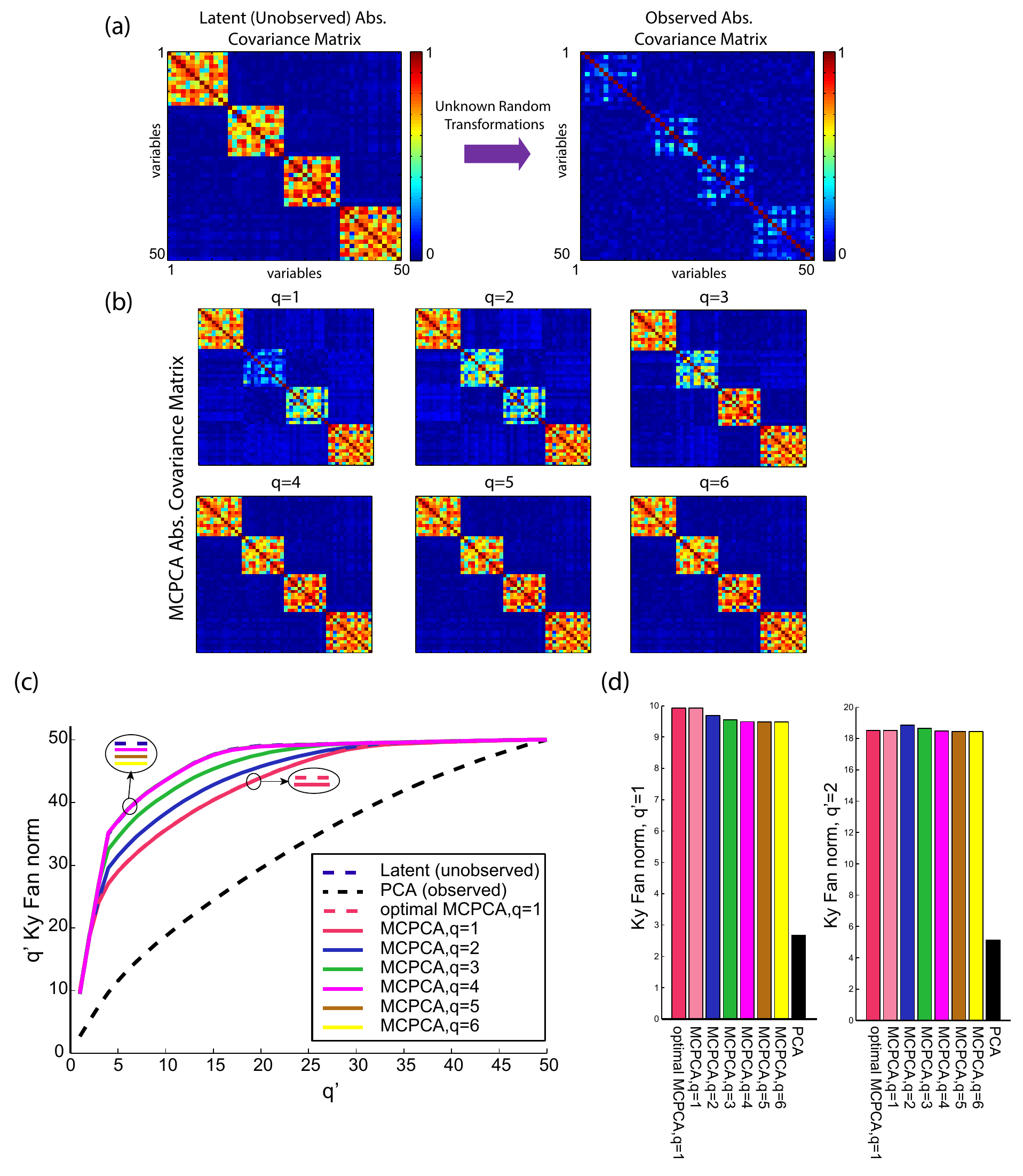}
\caption{(a) An illustration of latent, observed, and MCPCA absolute covariance matrices with different $q$ values. (b,c) An illustration of $q'$ Ky Fan norm of of latent, observed, and MCPCA covariance matrices with different $q$ values for $1\leq q'\leq p$ (panel b), and $q'=1,2$ (panel c). }
\label{fig:heatmap-p75}
\end{figure}

\section{MCPCA Applications to Synthetic and Real Data Sets}\label{sec:simulations}
\subsection{Synthetic Discrete Data}\label{subsec:synthetic-discrete-data}
First, we illustrate performance of MCPCA over simulated discrete data. We generate $n=1000$ independent samples from $p=50$ discrete variables whose covariance matrix is shown in Figure \ref{fig:heatmap-p75}-a (left panel). These samples are generated as discretized version of continuous jointly Gaussian samples. Alphabet sizes of variables (i.e., the number of quantization levels) are equal to 10. We then apply unknown random functions (with zero means and unit variances) to samples of each variable. The covariance matrix of observed samples (i.e., samples from transformed variables) is shown in Figure \ref{fig:heatmap-p75}-a (right panel). Owing to transformations of variables, the block diagonal structure of the latent covariance matrix has been faded in the observed one.

We apply the sample MCPCA Algorithm \ref{alg:sample-MCPCA} with parameter $q$ to the observed data matrix. We use 10 random initializations and 10 repeats of Algorithm \ref{alg:sample-MCPCA}. Figure \ref{fig:heatmap-p75}-b illustrates the covariance matrix computed by the MCPCA algorithm with parameter $1\leq q\leq 6$. MCPCA with $q=1$ highlights some of the block diagonal structure in the latent covariance matrix. MCPCA with larger $q$ recovers all the blocks. Note that the MCPCA algorithm aims to find a covariance matrix of transformed variables with the largest Ky Fan norm and is not tailored to infer a specific hidden structure in the data. Nevertheless inferring a low rank covariance matrix often captures such hidden structures in the data.

Figure \ref{fig:heatmap-p75}-c,d shows the $q'$ Ky Fan norm for the latent covariance matrix, for the observed covariance matrix (i.e., the PCA objective value), and for covariance matrices computed by MCPCA with different $q$ values. For $q=1$, Theorem \ref{thm:q=1} provides a globally optimal solution for the MCPCA optimization. We include that solution as well as the MCPCA solution computed in Algorithm \ref{alg:sample-MCPCA}. Figure \ref{fig:heatmap-p75}-c shows that the Ky Fan norm of covariance matrices computed by MCPCA are significantly larger than the one of the PCA. In Figure \ref{fig:heatmap-p75}-d, we show the $q'$ Ky Fan norm for $q'=1,2$ for different covariance matrices. Note that the method of Theorem \ref{thm:q=1} provides a globally optimal solution for $q'$ Ky Fan norm maximization when $q'=1$, while the MCPCA Algorithm \ref{alg:sample-MCPCA} provides a locally optimal solution. In this case (Figure \ref{fig:heatmap-p75}-d, the left panel), the gap between global and local optimal values is small. Moreover for the case of $q'=2$ (Figure \ref{fig:heatmap-p75}-d, the right panel), the MCPCA solution with parameter $q=2$ is outperforming other solutions. Finally in the case considered in Figure \ref{fig:heatmap-p75}-c,d, we observe that the Ky Fan norm of the covariance matrix computed by the MCPCA algorithm is not sensitive to parameter $q$.

\begin{figure}
\centering
  \includegraphics[width=0.7\linewidth]{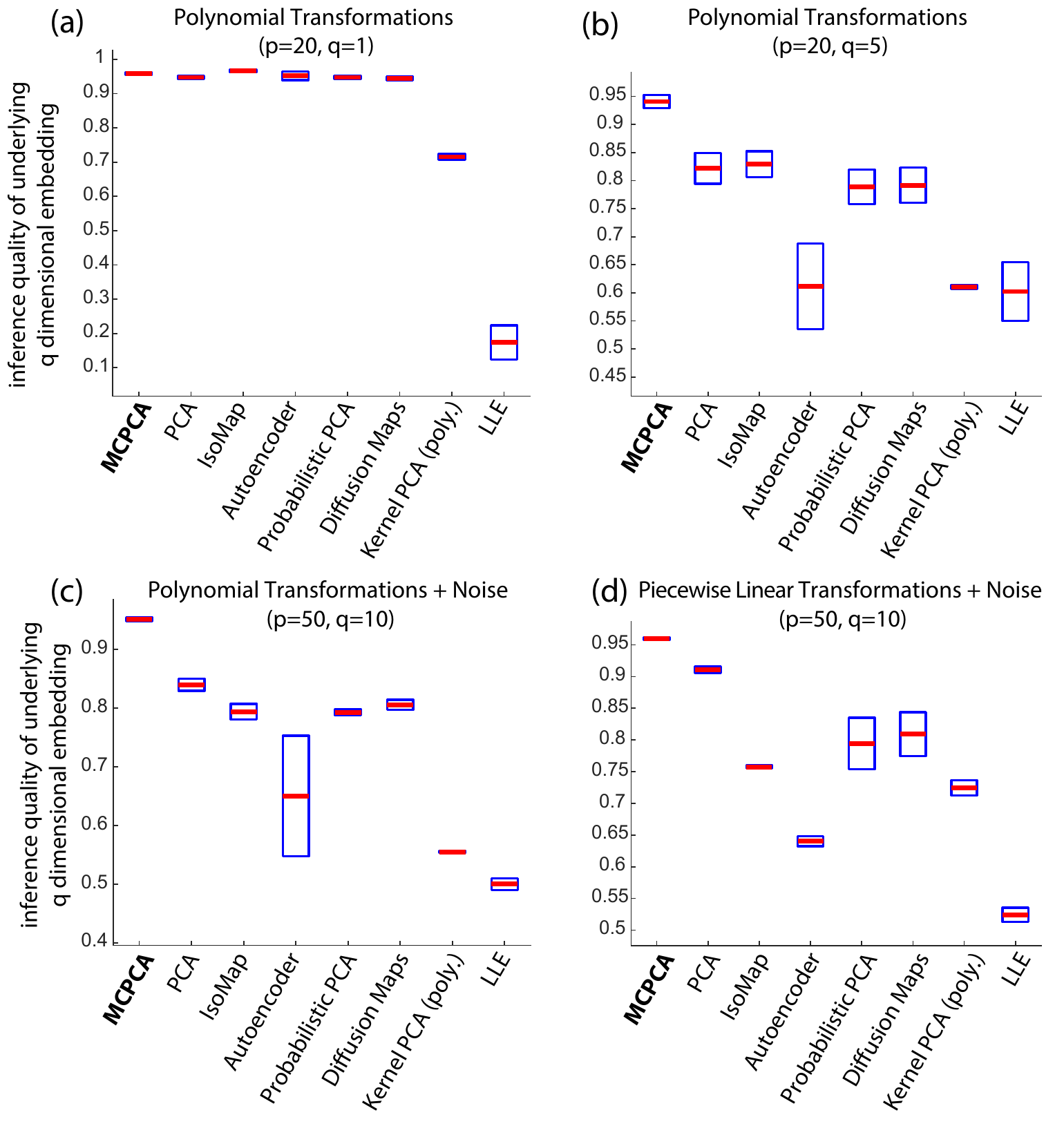}
\caption{Performance comparison of dimensionality reduction methods on synthetic datasets. The line in the middle of each box is the median result while the tops and bottoms of each box are the 25th and 75th percentiles of the results.}
\label{fig:cmp-all}
\end{figure}

\begin{figure}
\centering
  \includegraphics[width=0.9\linewidth]{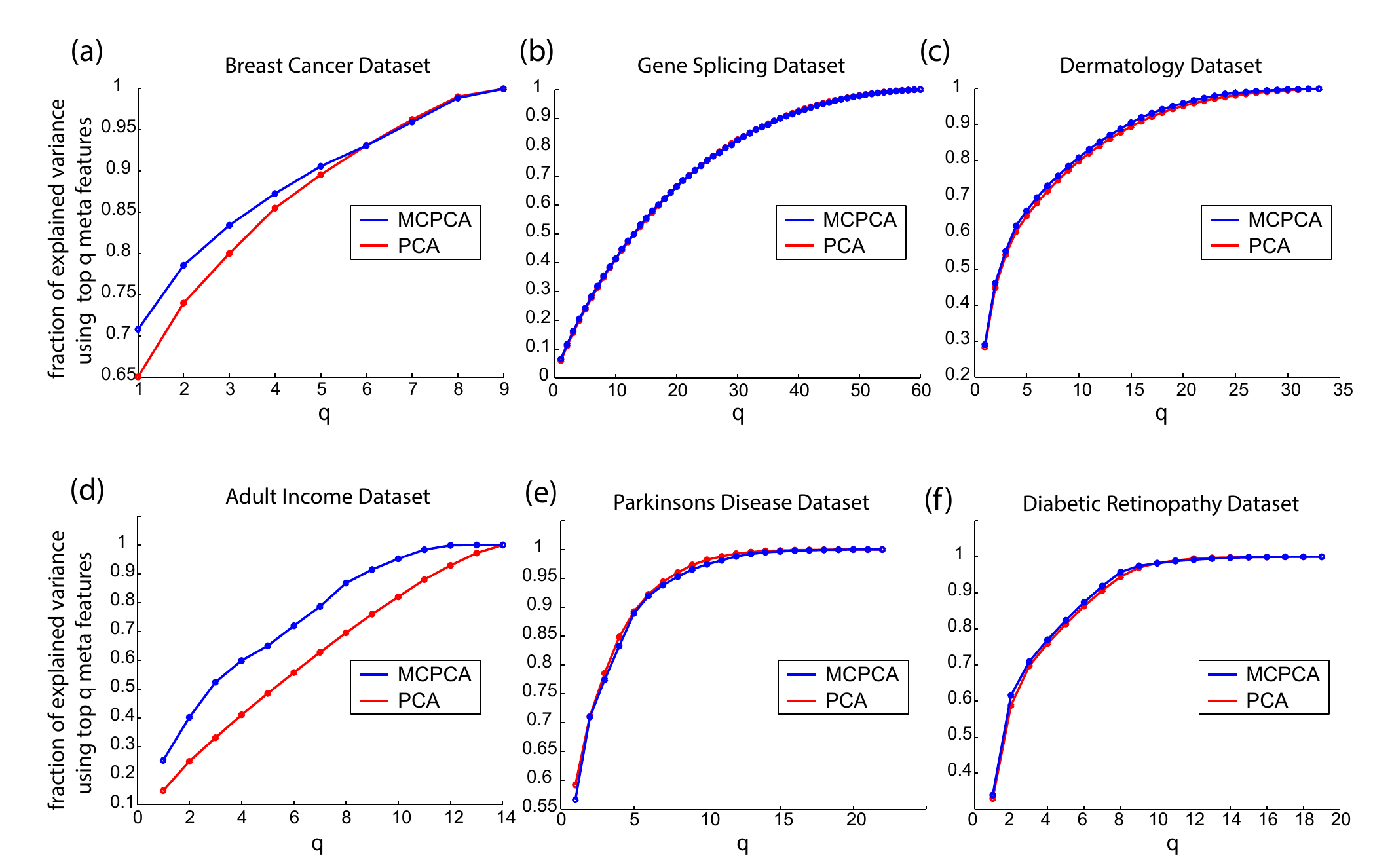}
\caption{This figure demonstrates the fraction of explained variance in six datasets using top q meta features computed by PCA and MCPCA.}
\label{fig:variance}
\end{figure}

\subsection{Synthetic Continuous Data}\label{subsec:synthetic-cont-data}
Next, we compare performance of different dimensionality reduction methods including MCPCA, PCA, Isomap, LLE, multilayer autoencoders (neural networks), kernel PCA, probabilistic PCA and diffusion maps on synthetic datasets. We assess the performance of different dimensionality reduction methods based on how much sample distances in the inferred and true low dimensional spaces match with each other. More precisely, let $\bX_{latent}\in\mathbb{R}^{n\times p}$ be a matrix whose rank is $q<p$. Let $d_{i,j}$ be the distance between sample $i$ and $j$ in the $q$ dimensional representation of $\bX_{latent}$. Let $\bN\in\mathbb{R}^{n\times p}$ be the noise matrix. Let $\bX\in\mathbb{R}^{n\times p}$ be the observed data matrix whose columns are transformations of columns of the matrix $\bX_{latent}+\bN$. These transformations are assumed to be continuous and bijective. Let $\tilde{d}_{i,j}$ be the distance between sample $i$ and $j$ in the inferred $q$ dimensional representation of $\bX$. We asses the performance of the dimensionality reduction method by computing the Spearman's rank correlation between $d_{i,j}$ and $\tilde{d}_{i,j}$ for $1\leq i,j\leq n$.

We generate $\bX_{latent}$ as $\bX_{latent}=\bU \bV^T$ where $\bU\in\mathbb{R}^{n\times q}$ and $\bV\in\mathbb{R}^{p\times q}$. Elements of $\bU$ and $\bV$ are generated according to a Gaussian distribution with zero mean and unit variance. In the noiseless case, $\bN$ is an all zero matrix. In the noisy case, elements of $\bN$ are generated according to a Gaussian distribution with zero mean and unit variance. We consider two types of transformations to generate columns of $\bX$ using columns of the matrix $\bX_{latent}+\bN$: (i) a polynomial transformation where for each variable we randomly select a transformation from the set $\{\bx,\bx^3,\bx^5\}$, and (ii) a piecewise linear transformation according to Definition \ref{def:piecewise-linear} where $w_{j+1}-w_{j}$ has an exponential distribution with parameter $100$. The positions of knots are chosen so that each bin has equal number of samples.

We use default parameters for different dimensionality reduction methods. IsoMap and LLE have a parameter $N_{ngbr}$ which determines the number of neighbors considered in their distance graphs. $N_{ngbr}$ is set to be 12. Moreover for the continuous data, MCPCA has a parameter $d$ which restricts the optimization to a set of piecewise linear functions with degree $d$. We set $d=10$. For other methods we use implementations of reference \cite{van2009dimensionality}. Experiments have been repeated 10 times in each case.

In Figure \ref{fig:cmp-all}-a we consider a relatively easy setup where $p=20$, $q=1$, transformation functions are polynomials, and there is no added noise to observed samples. In this setup, all methods except LLE and kernel PCA have good performance. Gaussian kernel PCA performed poorly in these experiments. Thus, we only illustrate performance of polynomial kernel PCA in this figure. It further highlights sensitivity of kernel PCA to the model setup. In Figure \ref{fig:cmp-all}-b we consider a similar setup to the one of panel (a) but we increase $q$ to be 5. MCPCA continues to have a good performance while the performance of other methods drop significantly. Next, we increase $p$ to 50 and $q$ to 10. We also add noise to observed samples as described above. MCPCA continues to have a good performance outperforming all other methods (Figure \ref{fig:cmp-all}-c). In Figure \ref{fig:cmp-all}-d we change nonlinear transformations from polynomials to piecewise linear functions compared to panel (c). Again, in this setup MCPCA outperforms all other methods. These experiments highlight robustness of MCPCA against model parameters and noise. Performance of other methods appears to be sensitive to these factors.

\begin{table}[t]
\centering
\resizebox{0.9\textwidth}{!}{\begin{minipage}{\textwidth}
  \begin{tabular}{ | c | c | c | c| c| }
    \hline
     \text{Data Set}& \# \text{of samples} $(n)$ &\# \text{of features} $(p)$ &\# of \text{of classes}& \text{class distribution}  \\ \hline
    \text{Breast Cancer} & 683 & 9 & 2 & (239,444) \\ \hline
    \text{Gene Splicing} & 3,175 & 60 & 2 & (1527,1648) \\ \hline
    \text{Dermatology} & 366 & 33 & 6 & (112,61,72,49,52,20) \\ \hline
    \text{Adult Income} & 30,162 & 14 & 2 & (7508,22654) \\ \hline
    \text{Parkinsons Disease} & 195 & 22 & 2 & (48,147) \\ \hline
    \text{Diabetic Retinopathy} & 1,151 & 19 & 2 & (540,611) \\ \hline
  \end{tabular}\caption{Properties of data sets considered in Section \ref{subsec:real-data}. \label{tab:dataset-prop}}
\end{minipage} }
\end{table}

\subsection{Real Data Analysis}\label{subsec:real-data}
Having illustrated effectiveness of MCPCA on synthetic datasets, we apply it to real datasets. We consider six data sets from the UCI machine learning repository data sets \cite{UCI}, namely breast cancer data set, gene splicing data set, dermatology data set, adult income data set, parkinsons disease data set, and diabetic retinopathy data set. These data sets have been chosen to span various types of input data. Some of them have discrete features, some have continuous features, while some have mixed discrete and continuous features. The number of samples ($n$) and the number of features ($p$) vary across these data sets. Samples in five of these data sets have binary labels while in one of them the number of sample classes is six. Basic properties of these data sets have been summarized in Table \ref{tab:dataset-prop}. Below we explain some of these  properties with more details:
\begin{itemize}
  \item The breast cancer data set has 683 individuals with breast cancer, among which 444 are benign and 239 are malignant (we remove 16 samples with missing values from the original data set.). Attributes in this data set include features such as clump thickness, uniformity of cell size, mitoses, etc. Values of these features are discrete in the set of $\{1,2,...,10\}$. For more information about this data set, see \cite{breast}.
  \item The gene splicing data set has 3,175 samples \footnote{We use the processed data provided in \url{http://www.cs.toronto.edu/~delve/data/datasets.html}}. Each sample is a 60 base pair subset of genome. The goal is to classify two types of splice junctions in DNA sequences: exon/intron (EI) or intron/exon (IE) sites.  Values of features are discrete in the set of $\{A,G,C,T\}$. For more information about this data set, see \cite{UCI}.
  \item The dermatology data set has 366 samples and 33 features (we ignore the age feature from the original data since it has missing values.). The classification of erythemato-squamous diseases is a difficult task in dermatology since they share clinical features of erythema and with similar scaling. This data set have samples with six diseases: psoriasis, seboreic dermatitis, lichen planus, pityriasis rosea, cronic dermatitis, and pityriasis rubra pilaris. The number of samples of each disease are 112, 61, 72, 49, 52, 20, respectively. Features include 12 clinical features and 21 histopathological features. Variables are discrete whose alphabet sizes are 2 (for one feature), 3 (for one feature), and 4 (for 31 features). For more information about this data set, see \cite{derm}.
  \item The adult income data set is the largest data set we consider in this section. It has 30,162 samples (after removing samples from the original training data with missing values.). The task is to classify individuals to two groups based on their income. This data set includes 22,654 individuals with income $\leq 50,000\$$ and $7,508$ individuals with income $>50,000\$$. Features include variables such as age, sex, race, education, work class, capital gain, capital loss, hours per week, etc. All features except one has fewer than 120 distinct alphabet values. For more information about this data set, see \cite{adult}.
  \item The parkinsons disease data set has 195 samples where 47 of them come from healthy individuals and 147 of them come from parkinsons patients. Each feature is a particular voice measure such as average vocal fundamental frequency, measures of variation in amplitude, measures of frequency variation, etc. Features are continuous with alphabet sizes ranges from 20 to 195. For more information about this data set, see \cite{parkinsons}.
  \item The diabetic retinopathy data set has 1,151 samples where 540 samples have no signs of the disease. The data contains 19 features extracted from the messidor image set to predict whether an image contains signs of diabetic retinopathy or not. The alphabet size of features range from 2 to 1,151. For more information on this data set, see \cite{diabetics}.
\end{itemize}

PCA and MCPCA aim to maximize the amount of explained variance in the data (or in the transformation of the data) using low dimensional features. PCA restricts its optimization to merely linear transformations while MCPCA considers a more general family of nonlinear transformation functions. More precisely, let $\bK_{(\phi_1,...,\phi_p)}\in\mathbb{R}^{p\times p}$ be the covariance matrix of transformations of variables. Then $\frac{1}{p}\sum_{r=1}^{q}\lambda_r(\bK_{\phi_1,...,\phi_p})$ is the fraction of explained variance in the transformation of the data using its optimal $q$ dimensional representation. We normalize features to have zero means and unit variances.

We perform a two-fold cross validation analysis: we choose half of the data uniformly randomly for training. Then we test performance of the methods in the remaining half of the data. In discrete data sets (i.e., breast cancer, gene splicing and dermatology data sets) we use sample MCPCA Algorithm \ref{alg:sample-MCPCA} to compute optimal transformations of features in the training data for each $q$ value. Then, we apply those transformations to the test data. In the adult income data set all features except one has fewer than 120 distinct alphabet values. For the only continuous feature in this data set we use $d=120$.  In continuous data sets (i.e., Parkinsons disease and diabetic retinopathy data sets) we use the procedure explained in Section \ref{subsec:comp-cont}. In these experiments $d=10$ is fixed. We repeat each experiment 10 times.

\begin{figure}
\centering
  \includegraphics[width=0.8\textwidth]{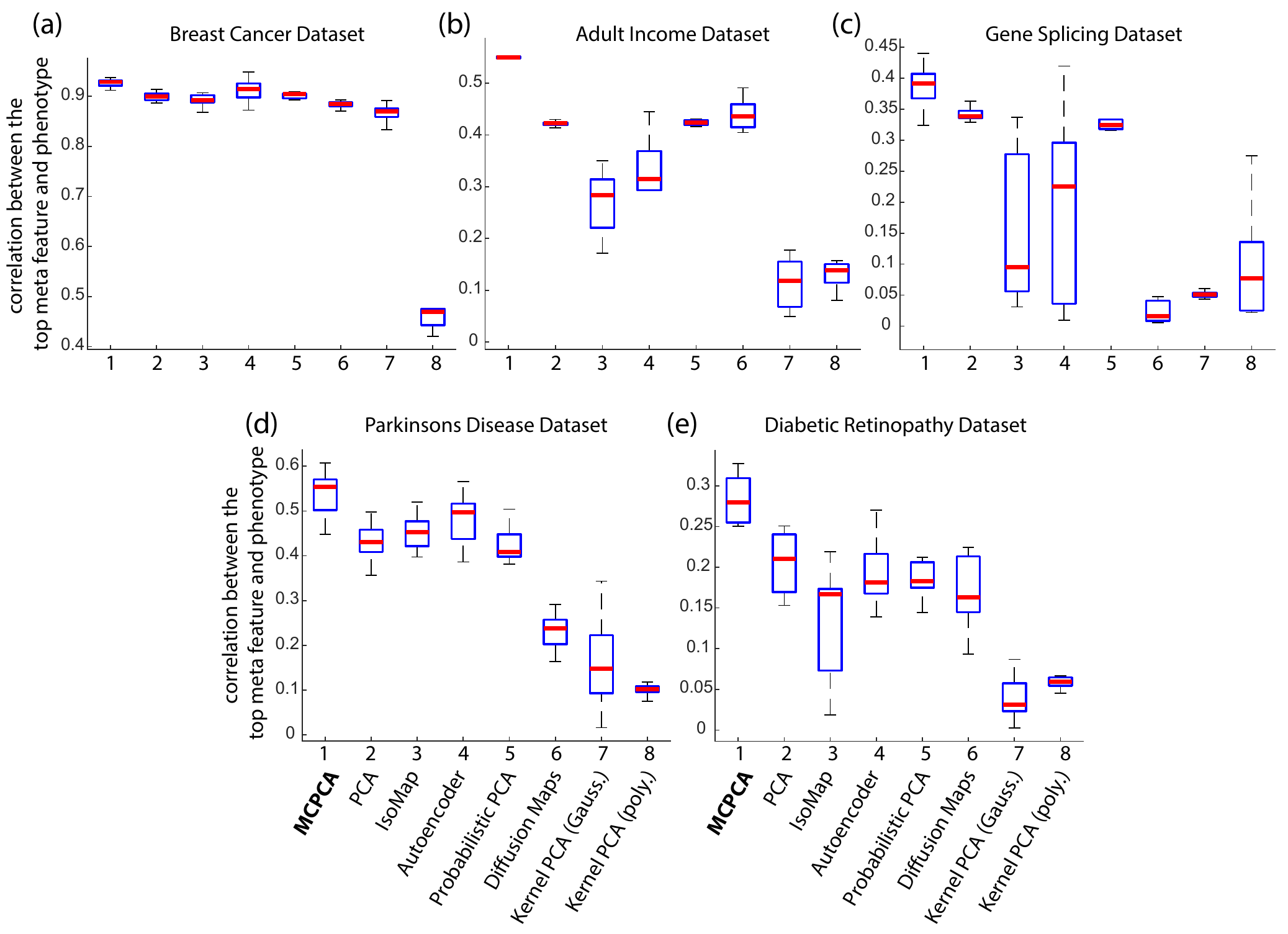}
\caption{This figure illustrates the correlation between the top meta feature and phenotype for five datasets with binary phenotypes. }
\label{fig:corr-cross}
\end{figure}

Figure \ref{fig:variance} shows the fraction of explained variance using top $q$ meta features computed by PCA and MCPCA in a two-fold cross validation analysis. In breast cancer and adult income datasets MCPCA significantly outperforms PCA for all values of $q$, while in other datasets their performance is comparable. The fact that MCPCA shows higher or comparable performance to PCA in holdout datasets indicates that MCPCA captures meaningful nonlinear correlations among features whenever they exist.

Next, we examine how predictive of phenotype extracted meta features are. Similarly to the previous experiment we use a two-fold cross validation analysis. We choose half of samples uniformly randomly to train the methods, and test their performance in the remaining half. We repeat each experiment 10 times. In continuous data sets, we consider $d\in\{1,5,10,15\}$. For the Isomap in the training phase we consider $N_{ngbr}\in\{10,15,20\}$. In the Isomap case, since the method does not have the so-called {\it parametric out-of-sample} property \cite{van2009dimensionality} (meaning that we cannot use the low dimensional embedding of the training data to compute a low dimensional embedding of the test data), we run the method on the test data using optimal parameters learned in the training step. This issue occurs in other nonlinear dimensionality reduction methods. In those cases we run the methods in the test data using their default parameters \cite{van2009dimensionality}.

Figure \ref{fig:corr-cross} shows the correlation between the top extracted meta feature using different dimensionality reduction methods and phenotype. The implementation of LLE crashed in these experiments, thus excluded from this figure. In all cases MCPCA consistently outperforms all other methods in different ranges of correlation between the meta feature and phenotype. For example, correlation between the meta feature and phenotype is high in the breast cancer dataset, is average in the adult income dataset, and is low in gene splicing and diabetic retinopathy datasets. Nevertheless, in all cases MCPCA shows a significant gain over all other methods.

\section{Discussion}
Here we introduced Maximally Correlated Principal Component Analysis (MCPCA) as a multivariate extension of maximal correlation and a generalization of PCA. MCPCA computes, possibly nonlinear, transformations of variables whose covariance matrix has the largest Ky Fan norm. MCPCA resolves two weaknesses of PCA by considering nonlinear correlations among features and being suitable for both continuous and categorical data. Although the MCPCA optimization is non-convex, we characterized its global optimizers for nonlinear functions of jointly Gaussian variables, and for categorical variables under some conditions. For general categorical variables, we proposed a block coordinate descend algorithm and showed its convergence to stationary points of the MCPCA optimization. Given the widespread applicability of PCA and the improved and robust performance of MCPCA compared to state-of-the-art dimensionality reduction methods, we expect the proposed method to find broad use in different areas of science. Moreover, techniques developed for efficiently optimizing feature transformations over a broad family of linear and nonlinear functions can be employed in several other statistical and machine learning problems such as nonlinear regression and deep learning.



\begin{thebibliography}{10}
\providecommand{\url}[1]{#1}
\csname url@samestyle\endcsname
\providecommand{\newblock}{\relax}
\providecommand{\bibinfo}[2]{#2}
\providecommand{\BIBentrySTDinterwordspacing}{\spaceskip=0pt\relax}
\providecommand{\BIBentryALTinterwordstretchfactor}{4}
\providecommand{\BIBentryALTinterwordspacing}{\spaceskip=\fontdimen2\font plus
\BIBentryALTinterwordstretchfactor\fontdimen3\font minus
  \fontdimen4\font\relax}
\providecommand{\BIBforeignlanguage}[2]{{%
\expandafter\ifx\csname l@#1\endcsname\relax
\typeout{** WARNING: IEEEtran.bst: No hyphenation pattern has been}%
\typeout{** loaded for the language `#1'. Using the pattern for}%
\typeout{** the default language instead.}%
\else
\language=\csname l@#1\endcsname
\fi
#2}}
\providecommand{\BIBdecl}{\relax}
\BIBdecl

\bibitem{pearson1895note}
K.~Pearson, ``Note on regression and inheritance in the case of two parents,''
  \emph{Proceedings of the Royal Society of London}, pp. 240--242, 1895.

\bibitem{neter1996applied}
J.~Neter, M.~H. Kutner, C.~J. Nachtsheim, and W.~Wasserman, \emph{Applied
  linear statistical models}.\hskip 1em plus 0.5em minus 0.4em\relax Irwin
  Chicago, 1996, vol.~4.

\bibitem{jolliffe2002principal}
I.~Jolliffe, \emph{Principal Component Analysis}.\hskip 1em plus 0.5em minus
  0.4em\relax Wiley Online Library, 2002.

\bibitem{steinwart2008support}
I.~Steinwart and A.~Christmann, \emph{Support vector machines}.\hskip 1em plus
  0.5em minus 0.4em\relax Springer Science \& Business Media, 2008.

\bibitem{hirschfeld1935connection}
H.~O. Hirschfeld, ``A connection between correlation and contingency,'' in
  \emph{Mathematical Proceedings of the Cambridge Philosophical Society},
  vol.~31, no.~04, 1935, pp. 520--524.

\bibitem{gebelein1941statistische}
H.~Gebelein, ``Das statistische problem der korrelation als variations-und
  eigenwertproblem und sein zusammenhang mit der ausgleichsrechnung,''
  \emph{ZAMM-Journal of Applied Mathematics and Mechanics}, vol.~21, no.~6, pp.
  364--379, 1941.

\bibitem{sarmanov1962maximum}
O.~Sarmanov, ``Maximum correlation coefficient (nonsymmetric case),''
  \emph{Selected Translations in Mathematical Statistics and Probability},
  vol.~2, pp. 207--210, 1962.

\bibitem{renyi1959measures}
A.~R{\'e}nyi, ``On measures of dependence,'' \emph{Acta {M}athematica
  {H}ungarica}, vol.~10, no.~3, pp. 441--451, 1959.

\bibitem{witsenhausen1975sequences}
H.~S. Witsenhausen, ``On sequences of pairs of dependent random variables,''
  \emph{SIAM Journal on Applied Mathematics}, vol.~28, no.~1, pp. 100--113,
  1975.

\bibitem{ahlswede1976spreading}
R.~Ahlswede and P.~G{\'a}cs, ``Spreading of sets in product spaces and
  hypercontraction of the markov operator,'' \emph{The Annals of Probability},
  pp. 925--939, 1976.

\bibitem{lancaster1957some}
H.~O. Lancaster, ``Some properties of the bivariate normal distribution
  considered in the form of a contingency table,'' \emph{Biometrika}, vol.~44,
  no. 1-2, pp. 289--292, 1957.

\bibitem{isomap}
J.~B. Tenenbaum, V.~De~Silva, and J.~C. Langford, ``A global geometric
  framework for nonlinear dimensionality reduction,'' \emph{science}, vol. 290,
  no. 5500, pp. 2319--2323, 2000.

\bibitem{LLE}
S.~T. Roweis and L.~K. Saul, ``Nonlinear dimensionality reduction by locally
  linear embedding,'' \emph{Science}, vol. 290, no. 5500, pp. 2323--2326, 2000.

\bibitem{lee2007nonlinear}
J.~A. Lee and M.~Verleysen, \emph{Nonlinear dimensionality reduction}.\hskip
  1em plus 0.5em minus 0.4em\relax Springer Science and Business Media, 2007.

\bibitem{hinton2006reducing}
G.~E. Hinton and R.~R. Salakhutdinov, ``Reducing the dimensionality of data
  with neural networks,'' \emph{Science}, vol. 313, no. 5786, pp. 504--507,
  2006.

\bibitem{scholkopf1997kernel}
B.~Sch{\"o}lkopf, A.~Smola, and K.-R. M{\"u}ller, ``Kernel principal component
  analysis,'' in \emph{International Conference on Artificial Neural
  Networks}.\hskip 1em plus 0.5em minus 0.4em\relax Springer, 1997, pp.
  583--588.

\bibitem{scholkopf1998nonlinear}
------, ``Nonlinear component analysis as a kernel eigenvalue problem,''
  \emph{Neural computation}, vol.~10, no.~5, pp. 1299--1319, 1998.

\bibitem{hoffmann2007kernel}
H.~Hoffmann, ``Kernel {PCA} for novelty detection,'' \emph{Pattern
  Recognition}, vol.~40, no.~3, pp. 863--874, 2007.

\bibitem{mika1998kernel}
S.~Mika, B.~Sch{\"o}lkopf, A.~J. Smola, K.-R. M{\"u}ller, M.~Scholz, and
  G.~R{\"a}tsch, ``Kernel {PCA} and de-noising in feature spaces.'' in
  \emph{NIPS}, vol.~11, 1998, pp. 536--542.

\bibitem{roweis1998algorithms}
S.~Roweis, ``{EM algorithms for PCA and SPCA},'' \emph{Advances in neural
  information processing systems}, pp. 626--632, 1998.

\bibitem{lafon2006diffusion}
S.~Lafon and A.~B. Lee, ``Diffusion maps and coarse-graining: A unified
  framework for dimensionality reduction, graph partitioning, and data set
  parameterization,'' \emph{IEEE transactions on pattern analysis and machine
  intelligence}, vol.~28, no.~9, pp. 1393--1403, 2006.

\bibitem{weinberger2004learning}
K.~Q. Weinberger, F.~Sha, and L.~K. Saul, ``Learning a kernel matrix for
  nonlinear dimensionality reduction,'' in \emph{Proceedings of the
  twenty-first international conference on machine learning}.\hskip 1em plus
  0.5em minus 0.4em\relax ACM, 2004, p. 106.

\bibitem{belkin2001laplacian}
M.~Belkin and P.~Niyogi, ``Laplacian eigenmaps and spectral techniques for
  embedding and clustering.'' in \emph{NIPS}, vol.~14, 2001, pp. 585--591.

\bibitem{donoho2003hessian}
D.~L. Donoho and C.~Grimes, ``Hessian eigenmaps: Locally linear embedding
  techniques for high-dimensional data,'' \emph{Proceedings of the National
  Academy of Sciences}, vol. 100, no.~10, pp. 5591--5596, 2003.

\bibitem{zhang2004principal}
Z.-y. Zhang and H.-y. Zha, ``Principal manifolds and nonlinear dimensionality
  reduction via tangent space alignment,'' \emph{Journal of Shanghai University
  (English Edition)}, vol.~8, no.~4, pp. 406--424, 2004.

\bibitem{sammon1969nonlinear}
J.~W. Sammon, ``A nonlinear mapping for data structure analysis,'' \emph{IEEE
  Transactions on computers}, vol.~18, no.~5, pp. 401--409, 1969.

\bibitem{van2009dimensionality}
L.~Van Der~Maaten, E.~Postma, and J.~Van~den Herik, ``Dimensionality reduction:
  a comparative review,'' \emph{Journal of Machine Learning Research}, vol.~10,
  pp. 66--71, 2009.

\bibitem{feizi2015network}
S.~Feizi, A.~Makhdoumi, K.~Duffy, M.~Kellis, and M.~Medard, ``Network maximal
  correlation,'' \emph{arXiv preprint arXiv:1606.04789}, 2015.

\bibitem{beigi2015duality}
S.~Beigi and A.~Gohari, ``On the duality of additivity and tensorization,''
  \emph{arXiv preprint arXiv:1502.00827}, 2015.

\bibitem{ando1989majorization}
T.~Ando, ``Majorization, doubly stochastic matrices, and comparison of
  eigenvalues,'' \emph{Linear Algebra and Its Applications}, vol. 118, pp.
  163--248, 1989.

\bibitem{boyd2004convex}
S.~P. Boyd and L.~Vandenberghe, \emph{Convex optimization}.\hskip 1em plus
  0.5em minus 0.4em\relax Cambridge university press, 2004.

\bibitem{bapat1985majorization}
R.~Bapat and V.~Sunder, ``On majorization and schur products,'' \emph{Linear
  algebra and its applications}, vol.~72, pp. 107--117, 1985.

\bibitem{tseng2001convergence}
P.~Tseng, ``Convergence of a block coordinate descent method for
  nondifferentiable minimization,'' \emph{Journal of optimization theory and
  applications}, vol. 109, no.~3, pp. 475--494, 2001.

\bibitem{UCI}
K.~Bache and M.~Lichman, ``{UCI} machine learning repository,'' 2013.

\bibitem{breast}
O.~Mangasarian and W.~Wolberg, ``Cancer diagnosis via linear programming,''
  \emph{SIAM News}, vol.~23, no.~5, 1990.

\bibitem{derm}
H.~A. G{\"u}venir, G.~Demir{\"o}z, and N.~Ilter, ``Learning differential
  diagnosis of erythemato-squamous diseases using voting feature intervals,''
  \emph{Artificial intelligence in medicine}, vol.~13, no.~3, pp. 147--165,
  1998.

\bibitem{adult}
R.~Kohavi, ``Scaling up the accuracy of naive-bayes classifiers: a
  decision-tree hybrid,'' in \emph{Proceedings of the Second International
  Conference on Knowledge Discovery and Data Mining}, 1996.

\bibitem{parkinsons}
M.~A. Little, P.~E. McSharry, E.~J. Hunter, J.~Spielman, L.~O. Ramig
  \emph{et~al.}, ``Suitability of dysphonia measurements for telemonitoring of
  parkinson's disease,'' \emph{IEEE transactions on biomedical engineering},
  vol.~56, no.~4, pp. 1015--1022, 2009.

\bibitem{diabetics}
B.~Antal and A.~Hajdu, ``An ensemble-based system for automatic screening of
  diabetic retinopathy,'' \emph{Knowledge-Based Systems}, vol.~60, pp. 20--27,
  2014.

\end{thebibliography}
\end{document}